\def\year{2020}\relax
\documentclass[letterpaper]{article} % DO NOT CHANGE THIS
\usepackage{aaai20}  % DO NOT CHANGE THIS
\usepackage{times}  % DO NOT CHANGE THIS
\usepackage{helvet} % DO NOT CHANGE THIS
\usepackage{courier}  % DO NOT CHANGE THIS
\usepackage[hyphens]{url}  % DO NOT CHANGE THIS
\usepackage{graphicx} % DO NOT CHANGE THIS
\def\UrlFont{\rm}  % DO NOT CHANGE THIS
\usepackage{graphicx}  % DO NOT CHANGE THIS
\usepackage{adjustbox}
\usepackage{microtype}
\usepackage{svg}
\usepackage{mathrsfs}

\usepackage[utf8]{inputenc} % allow utf-8 input
\usepackage{hyperref}       % hyperlinks
\usepackage{url}            % simple URL typesetting
\usepackage{booktabs}       % professional-quality tables
\usepackage{amsfonts}       % blackboard math symbols
\usepackage{nicefrac}       % compact symbols for 1/2, etc.
\usepackage{microtype}      % microtypography

\usepackage{graphicx} 
\usepackage{tikz}
\usepackage{amsmath}

\usepackage{booktabs}
\catcode`\@ = 11
\newdimen\@InsertBoxMargin
\newcount\@numlines    % sum: m_1+...+m_n
\newcount\@linesleft   % counter used when reading lines of \ParShape
\def\ParShape{%
    \@numlines = 0
    \def\@parshapedata{ }% here we'll collect data for plain \parshape
    \afterassignment\@beginParShape
    \@linesleft
}%
\def\@beginParShape{%
    \ifnum \@linesleft = 0
      \let\@whatnext = \@endParShape
    \else
      \let\@whatnext = \@readnextline
    \fi
    \@whatnext
}%
\def\@endParShape{%
    \global\parshape = \@numlines \@parshapedata
}%
\def\@readnextline#1 #2 #3 {% #1 #2 #3 are: m_i, leftskip_i, rightskip_i
    \ifnum #1 > 0
      \bgroup  % I want to keep changes of \dimen0 and \count0 local
        \dimen0 = \hsize
        \advance \dimen0 by -#2  % \parshape requires left skip and
        \advance \dimen0 by -#3  % _length_of_line_ (not right skip!)
        \count0 = 0
        \loop
          \global\edef\@parshapedata{%
            \@parshapedata    % add to \@parshapedata:
            #2                % left skip
            \space            % a space
            \the\dimen0       % length of line
            \space            % another space
          }%
          \advance \count0 by 1
          \ifnum \count0 < #1
        \repeat
      \egroup
      \advance \@numlines by #1
    \fi
    \advance \@linesleft by -1
    \@beginParShape
}%
\newbox\@boxcontent     % box containing the picture to be inserted
\newcount\@numnormal    % number of leading lines to typeset normally
\newdimen\@framewidth   % width of the frame
\newdimen\@wherebottom  % position of frame's bottom
\newif\if@byframe       % true if we are just beside the frame
\@byframefalse
\def\InsertBoxC#1{%
  \leavevmode
  \vadjust{
    \vskip \@InsertBoxMargin
    \hbox to \hsize{\hss#1\hss}
    \vskip \@InsertBoxMargin
  }%
}%
\def\InsertBoxL#1#2{%
  \@numnormal = #1
  \setbox\@boxcontent = \hbox{#2}%
  \let\@side = 0
  \futurelet \@optionalparameter \@InsertBox
}
\def\InsertBoxR#1#2{%
  \@numnormal = #1
  \setbox\@boxcontent = \hbox{#2}%
  \let\@side = 1
  \futurelet \@optionalparameter \@InsertBox
}%
\def\@InsertBox{%
  \ifx \@optionalparameter [
    \let\@whatnext = \@@InsertBoxCorrection
  \else
    \let\@whatnext = \@@InsertBoxNoCorrection
  \fi
  \@whatnext
}%
\def\@@InsertBoxCorrection[#1]{%
  \ifx \@side 0
    \@@InsertBox{#1}{0}{{\the\@framewidth} 0cm}%
  \else
    \@@InsertBox{#1}{1}{0cm {\the\@framewidth}}%
  \fi
}%
\def\@@InsertBoxNoCorrection{%
  \@@InsertBoxCorrection[0]%
}%
\def\@@InsertBox#1#2#3{%
  \MoveBelowBox
  \@byframetrue
  \@wherebottom = \baselineskip
  \multiply \@wherebottom by \@numnormal
  \advance \@wherebottom by 2\@InsertBoxMargin
  \advance \@wherebottom by \ht\@boxcontent
  \advance \@wherebottom by \pagetotal
  \ifdim \pagetotal = 0cm
    \advance \@wherebottom by -\baselineskip  % ^ reduction
  \fi
  \advance \@wherebottom by #1\baselineskip
  \@framewidth = \wd\@boxcontent
  \advance \@framewidth by \@InsertBoxMargin
  \bgroup  % to keep changes of \dimen0 local
    \ifdim \pagetotal = 0cm
      \dimen0 = \vsize
    \else
      \dimen0 = \pagegoal
    \fi
    \ifdim \@wherebottom > \dimen0
      \immediate\write16{+--------------------------------------------------------------+}%
      \immediate\write16{| The box will not fit in the page. Please, re-edit your text. |}%
      \immediate\write16{+--------------------------------------------------------------+}%
      \vrule width \overfullrule
    \fi
  \egroup
  \prevgraf = 0
  \vbox to 0cm{%
    \dimen0 = \baselineskip
    \multiply \dimen0 by \@numnormal
    \advance \dimen0 by -\baselineskip
    \setbox0 = \hbox{y}%
    \vskip \dp0
    \vskip \dimen0
    \vskip \@InsertBoxMargin
    \ifnum #2 = 1
      \vtop{\noindent \hbox to \hsize{\hss \box\@boxcontent}}%
    \else
      \vtop{\noindent \box\@boxcontent}%
    \fi
    \vss
  }%
  \vglue -\parskip
  \vskip -\baselineskip
  \everypar = {%
    \ifdim \pagetotal < \@wherebottom
      \bgroup  % to keep some changes local
        \dimen0 = \@wherebottom
        \advance \dimen0 by -\pagetotal
        \divide \dimen0 by \baselineskip
        \count1 = \dimen0
        \advance \count1 by 1
        \advance \count1 by -\@numnormal
        \ifnum #2 = 1
          \ParShape = 3
                      {\the\@numnormal}   0cm   0cm
                      {\the\count1}       0cm   {\the\@framewidth}
                      1                   0cm   0cm
        \else
          \ParShape = 3
                      {\the\@numnormal}   0cm                  0cm
                      {\the\count1}       {\the\@framewidth}   0cm
                      1                   0cm                  0cm
        \fi
      \egroup
    \else
      \@restore@    % it's time to end everything
    \fi
  }%
  \def\par{%
      \endgraf
      \global\advance \@numnormal by -\prevgraf
      \ifnum \@numnormal < 0
        \global\@numnormal = 0
      \fi
      \prevgraf = 0
  }%
}%
\def\MoveBelowBox{%
  \par
  \if@byframe
    \global\advance \@wherebottom by -\pagetotal
    \ifdim \@wherebottom > 0cm
      \vskip \@wherebottom
    \fi
    \@restore@
  \fi
}%
\def\@restore@{%
    \global\@wherebottom = 0cm
    \global\@byframefalse
    \global\everypar = {}%
    \global\let \par = \endgraf
    \global\parshape = 1 0cm \hsize
}%
\ifx \documentclass \@Dont@Know@What@It@Is@
\else
  \let \pageno = \c@page
\fi

\catcode`\@ = 12

\usepackage{multirow}

\usepackage{algorithmic}
\usepackage[linesnumbered, ruled]{algorithm2e} 

\usepackage{mwe}
\usepackage{xr}

\usepackage{amsmath,amsthm}
\usepackage{environ}
\usepackage{etoolbox}
\usepackage{comment}

\makeatletter
\newtheorem*{rep@theorem}{\rep@title}
\newcommand{\newreptheorem}[2]{%
\newenvironment{rep#1}[1]{%
 \def\rep@title{#2 \ref{##1}}%
 \begin{rep@theorem}}%
 {\end{rep@theorem}}}
\makeatother

\newtheorem{theorem}{Theorem}
\newtheorem{assumption}{Assumption}
\newtheorem{conjecture}{Conjecture}
\newtheorem{lemma}{lemma}

\newreptheorem{proposition}{Proposition}
\newreptheorem{theorem}{Theorem}
\newreptheorem{lemma}{Lemma}
\newreptheorem{observation}{Observation}

\usepackage{changepage}

\newtheorem{observation}{Observation}
\theoremstyle{definition}
\newtheorem{definition}{Definition}%[section]

\usepackage{amsmath,amsfonts,bm}

\def\eqref#1{equation~\ref{#1}}
\def\1{\bm{1}}

\def\rvb{{\mathbf{b}}}

\def\rvv{{\mathbf{v}}}
\def\rvw{{\bm{\beta}}}
\def\rvx{{\mathbf{x}}}

\def\vzero{{\bm{0}}}

\def\vmu{{\bm{\mu}}}

\def\vbb{{\bm{b}}}
\def\vb{{\bm{u}}}

\def\vk{{\bm{k}}}

\def\vu{{\bm{u}}}

\def\vx{{\bm{x}}}
\def\vy{{\bm{y}}}
\def\vz{{\bm{z}}}

\def\mK{{\bm{K}}}

\def\mY{{\bm{Y}}}

\DeclareMathAlphabet{\mathsfit}{\encodingdefault}{\sfdefault}{m}{sl}
\SetMathAlphabet{\mathsfit}{bold}{\encodingdefault}{\sfdefault}{bx}{n}

\def\gN{{\mathcal{N}}}
\def\gO{{\mathcal{O}}}
\def\gP{{\mathcal{P}}}

\def\gX{{\mathcal{X}}}

\newcommand{\E}{\mathbb{E}}
\newcommand{\W}{\mathbb{W}}

\newcommand{\R}{\mathbb{R}}

\def\calD{{\mathcal{D}}}
\def\calF{{\mathcal{F}}}

\def\calB{{\mathcal{B}}}
\def\calN{{\mathcal{N}}}

\def\calX{{\mathcal{X}}}

\let\log\relax
\DeclareMathOperator{\log}{ln}

\def\P{{\mathbb{P}}}

\def\Q{{\mathbb{Q}}}

\def\I{{\mathbb{I}_d}}
\def\ind{{\bm{1}}}
\def\R{{\mathbb{R}}}

\def\Pr{{\gP_{\pi_{0,1}}}}
\def\hPr{{\hat{\gP}_{\pi_{0,1}}}}

\def\hphi{{\hat{\phi}}}

\newcommand{\KL}{D_{\mathrm{KL}}}

\def\vec{{\text{vec}}}

\DeclareMathOperator*{\argmax}{arg\,max}
\DeclareMathOperator*{\argmin}{arg\,min}
\DeclareMathOperator*{\arginf}{arg\,inf}

\usepackage[mathlines]{lineno} %% <- with [mathlines] to number lines in equations
\newcommand*\linenomathpatch[1]{%
  \cspreto{#1}{\linenomath}%
  \cspreto{#1*}{\linenomath}%
  \csappto{end#1}{\endlinenomath}%
  \csappto{end#1*}{\endlinenomath}%
}
\newcommand*\linenomathpatchAMS[1]{%
  \cspreto{#1}{\linenomathAMS}%
  \cspreto{#1*}{\linenomathAMS}%
  \csappto{end#1}{\endlinenomath}%
  \csappto{end#1*}{\endlinenomath}%
}

\expandafter\ifx\linenomath\linenomathWithnumbers
  \let\linenomathAMS\linenomathWithnumbers
  \patchcmd\linenomathAMS{\advance\postdisplaypenalty\linenopenalty}{}{}{}
\else
  \let\linenomathAMS\linenomathNonumbers
\fi

\linenomathpatch{equation}
\linenomathpatchAMS{gather}
\linenomathpatchAMS{multline}
\linenomathpatchAMS{align}
\linenomathpatchAMS{alignat}
\linenomathpatchAMS{flalign}

\makeatletter
\patchcmd{\mmeasure@}{\measuring@true}{
  \measuring@true
  \ifnum-\linenopenaltypar>\interdisplaylinepenalty
    \advance\interdisplaylinepenalty-\linenopenalty
  \fi
  }{}{}
\makeatother

\title{Solving Schrödinger Bridges via Maximum
Likelihood }
\author{ \Large \textbf{Francisco Vargas\textsuperscript{\rm 1} , Pierre Thodoroff\textsuperscript{\rm 1} , Austen Lamacraft\textsuperscript{\rm 2} , Neil D. Lawrence\textsuperscript{\rm 1} }\\ % All authors must be in the same font size and format. Use \Large and \textbf to achieve this result when breaking a line

\textsuperscript{\rm 1} Department of Computer Science, Cambridge University\\ %If you have multiple authors and multiple affiliations
\textsuperscript{\rm 2} Department of Physics Cambridge University\\

\{fav25,pt440,al200,ndl21\}@cam.ac.uk % email address must be in roman text type, not monospace or sans serif

}

\begin{document}

\newcommand{\citet}[1]{\citeauthor{#1}\shortcite{#1}}
\newcommand{\citep}{\cite}

\maketitle

\begin{abstract}
The Schrödinger bridge problem (SBP) finds the most likely stochastic evolution between two probability distributions given a prior stochastic evolution. As well as applications in the natural sciences, problems of this kind have important applications in machine learning such as dataset alignment and hypothesis testing. Whilst the theory behind this problem is relatively mature, scalable numerical recipes to estimate the Schrödinger bridge remain an active area of research. Our main contribution is the proof of equivalence between solving the SBP and an autoregressive maximum likelihood estimation objective. This formulation circumvents many of the challenges of density estimation and enables direct application of successful machine learning techniques. We propose a numerical procedure to estimate SBPs using Gaussian process and demonstrate the practical usage of our approach in numerical simulations and experiments.
\end{abstract}

\section{Introduction}

Analysis of cross-sectional data is ubiquitous in machine learning and science. Temporal data are typically sampled at discrete intervals due to technological or physical constraints. This means information between time points is lost. This motivates the need to model the stochastic evolution of a process between sampled time points. The classical Schrödinger bridge problem \cite{schrodinger1931uber,schrodinger1932theorie} finds the most likely stochastic process that evolves a distribution $\pi_0(\vx)$ to another distribution $\pi_1(\vy)$ consistently with a pre-specified Brownian motion. We consider a more general dynamical Schrödinger bridge problem for \emph{any} pre-specified diffusion prior. Practically, this generalization allows us to exploit domain knowledge, e.g. oceanic and atmospheric flows might be interpolated from empirical measurements using previously established dynamics as priors.

In the classical set up numerical approaches to solve the Schrödinger bridge are mainly based on the Sinkhorn-Knopp algorithm \cite{sinkhorn1967concerning}. However, extending those algorithms to more general diffusion priors and marginals requires complex adaptations. We introduce an iterative proportional maximum likelihood (IPML) algorithm to solve the general Schrödinger bridge problem. The IPML algorithm also obtains a good approximation for the dynamics of the underlying physical process that solves the SBP. This contrasts to previous approaches \cite{cuturi2013sinkhorn,feydy2019interpolating,chizat2020faster} that estimate the value that extremises the SBP objective. Practically, this means we obtain physically interpretable solutions that we can leverage for downstream tasks. 
%\end{paracol}
\begin{figure}[t!]
    \centering
    \includegraphics[width=0.85\columnwidth]{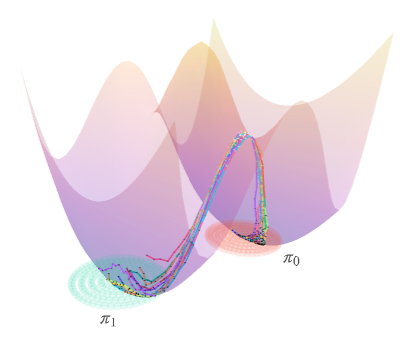}
    \caption{  Learned SBP trajectories in the double well experiment of Section \ref{sec:double}, with prior $\Q_0^\gamma$ expressed in terms of an energy landscape $U(x,y)$ as $d \rvx(t) = -\nabla_{\vx}U (\rvx(t)) + \gamma d\rvw(t)$}
    \label{fig:well_tr}
    % \caption{Learned SBP trajectories in the double well experiment of Section \ref{sec:double}, with prior $\Q_0^\gamma$ expressed in terms of an energy landscape $U(x,y)$ as $d \rvx(t) = -\nabla_{\vx}U (\rvx) + \gamma d\rvw(t)$}
    % \caption{}
\end{figure}
\begin{figure}
    \centering
    \includegraphics[width=0.8\columnwidth]{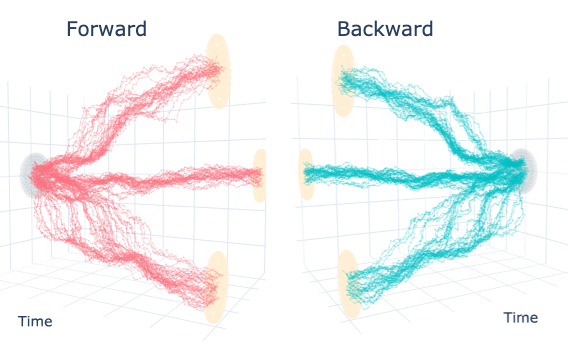}
    % \caption{Fitted SBP on the start and end data slices with Brownian prior on single cell human Embryo data. Observations depicted as red point clouds. See section \ref{sec:cells} for experimental details.}
    % \caption{}
 
    \caption{  Forwards and backwards diffusion of learned SBP between unimodal and multimodal boundary distributions (see section \ref{sec:simple} for experimental details)} \label{fig:well_and_forward}
\end{figure}
%\begin{paracol}{2}
%\linenumbers
%\switchcolumn
IPML's inspiration is from \emph{probabilistic numerics} (PN) \cite{hennig2015probabilistic}. We combine a PN styled formulation with the iterative proportional fitting procedure (IPFP) \cite{kullback1968probability,ruschendorf1995convergence}. The algorithm iteratively simulates trajectories that converge to the SBP. We prove  that IPML converges in probability at each iteration. Our numerical experiments show the algorithm can be implemented with Gaussian process (GP) models of the drifts. GPs allow us to incorporate functional prior information. We demonstrate the practical use of our algorithm on real-world embryoid cells data (see Figure \ref{fig:single_cel_trajectories}) by quantitatively and qualitatively comparing the performance of our algorithm to state-of-the-art deep learning methods and optimal transport techniques.  To summarise the main contributions of our work are:
\begin{itemize}
    \item We recast the iterations of the dynamic IPFP algorithm as a regression based maximum likelihood objective. This is formalised in Theorem \ref{theorem:mle-rev} and Observation \ref{obs:half}. This differs to prior approaches such as \cite{pavon2018data} where their maximum likelihood formulation solves a density estimation problem. This allows the application of many regression methods from machine learning that scale well to high dimensional problems. Note that this formulation can be parametrised with any method we chose GPs however neural networks would also be well suited.
    \item We solve the aforementioned regression objectives using  GPs \cite{williams2006gaussian} motivated by the connection between the drift of stochastic differential equations and  GPs \cite{ruttor2013approximate}.
    \item We provide a conceptual and empirical comparison with the approach by \cite{pavon2018data} and detail why the density estimation formulation in \cite{pavon2018data} scales poorly with dimension.
    % \item Finally we re-implement the approach by \cite{pavon2018data} and compare approaches across a series of numerical experiments. Furthermore we empirically show how our approach works well in dimensions higher than 2.
\end{itemize}

\section{Technical Background}
% \begin{wrapfigure}{r}{0.5\textwidth}
% % \begin{minipage}{0.5\textwidth}
% % \begin{figure} [t!]
%     \begin{center}
%     \includegraphics[width=0.47\textwidth]{images/reverse_forward.png}
%     \end{center}
%     \caption{Forwards and backwards diffusion of learned SBP between unimodal and multimodal boundary distributions (see section \ref{sec:simple} for experimental details).}
%     \label{fig:forback}
% % \end{figure}
% % \end{minipage}
% \end{wrapfigure}
Our solution has three components. (1) We reformulate the SBP as a dynamical system giving a stochastic differential equation (SDE) with initial value (IV) and final value (FV) constraints \cite{oksendal2003stochastic}. (2) We reverse the system, reformulating the FV constraint as an IV constraint. Both IV problems are solved through a stochastic control formulation (Section \ref{control}). (3) We iterate the IV and FV constrained problems to converge to the full boundary value SDE.

% \begin{figure}[t!]
%     \centering
%     \includegraphics[scale=0.5]{images/well-illustration.png}
%     \caption{Learned SBP trajectories in the double well experiment of Section \ref{sec:double}, with prior $\Q_0^\gamma$ expressed in terms of an energy landscape $U(x,y)$ as $d \rvx(t) = -\nabla_{\vx}U (\rvx) + \gamma d\rvw(t)$}
%     \label{fig:well_illustr}
% \end{figure}

\subsection{Dynamic Formulation}

The dynamic version of the Schrödinger bridge is written in terms of measures over the space of trajectories, that describe the stochastic dynamics defined over the unit interval.

\begin{definition}
    (Dynamic Schrödinger problem) The dynamic Schrödinger problem is given by
    \begin{linenomath}
    \begin{equation}
        \inf_{\Q \in \calD(\pi_0, \pi_1)} \KL\left(\Q \big|\big| \Q_0^{\gamma}\right),
    \end{equation}
    \end{linenomath}
    where $\Q\!\!\in\!\!\calD(\pi_0, \pi_1)$ is a  measure with prescribed marginals of $\pi_0, \pi_1$ at times $0, 1$ that is $(X_0)_{\#}\Q=\pi_0$ and $(X_1)_{\#}\Q=\pi_1$.  $\Q_0^\gamma$ is a drift augmented Brownian motion with a scalar volatility $\gamma$ acting as the prior, see Figure \ref{fig:well_tr}. Traditionally $\Q_0^{\gamma}\!:=\!\!\W^{\gamma}$ is a Wiener measure with volatility $\gamma$, however we consider the general setting in this work. 
\end{definition}

The prior $\Q_0^{\gamma}$ can be written as a solution  of the diffusion:
\begin{linenomath}
\begin{equation}\label{eq:sde1}
    d\rvx(t) = \vbb_0^{+}(\rvx(t),t) dt + \sqrt{\gamma} d\rvw^{+}(t), \;\;\; \rvx(0) \sim \pi_0^{ \Q_0^{\gamma}}.
\end{equation}
\end{linenomath}

From here we use $\vbb_0$ to denote the drift of the prior $\Q_0^{\gamma}$. Note that a finite KL implies $\Q$ and $\Q_0^\gamma$  are both Itô SDEs with volatility $\gamma$. The diffusion coefficient $\gamma$ is a time homogeneous constant and for the KL-divergence to be finite the process $\Q$ must have $\gamma$ as its diffusion coefficient.

\subsubsection{Time Reversal of Diffusions}

When sampling solutions, IV constraints are trivially solved by initialisation of samples, but FV constraints are more problematic. However, if we can reformulate the SBP as a \emph{time reversed diffusion} the FV constraint becomes an IV constraint (Section \ref{need_time_rev}). Here we review how the diffusion is reversed. 
We reverse time (see Figure \ref{fig:well_and_forward}) in the random variable $\rvx(t)$  described by the diffusion in Equation \ref{eq:sde1} such that $\rvx^{-}(t) = \rvx(1-t)$. The reverse time diffusion $\rvx^{-}(t)$ is also an Itô process. For a modern application of time reversal in Machine Learning see \cite{song2020score}.

\begin{lemma} \cite{nelson1967dynamical} 
If $\rvx^+(t)$ obeys the SDE
\begin{linenomath}
\begin{equation*}
    d\rvx^+(t)= \vb^{+}(\rvx^+(t),t) dt + \sqrt{\gamma} d\rvw^{+}(t), \;\;\; \rvx^+(0) \sim \pi_0^{ \Q_0^{\gamma}},
\end{equation*}
\end{linenomath}
then $\rvx^{-}(t) = \rvx^+(1-t)$ obeys
\begin{linenomath}
\begin{equation*}
    d\rvx^{-}(t) = \vb^{-}(\rvx^{-}(t),t) dt + \sqrt{\gamma} d\rvw^{-}(t), \;\;\; \rvx^{-}(0) = \rvx^+(1).
\end{equation*}
\end{linenomath}
 where $\rvw^{-}(t)$ is a Brownian motion adapted to the reverse filtration $(\calF^{-}_i)_{i\in T}$, that is $\calF^{-}_t \!\!\subseteq \calF^{-}_s, \; s \!\leq\! t$%\footnote{This simply means $\rvw^-(t)$ is only aware about the future.}
 
 Furthermore, the dual drift $\vb^{-}(\rvx,t)$ satisfies Nelson's duality relation:
 \begin{linenomath}
\begin{equation}\label{eq:nelson}
      {\vb^{+}(\rvx,t) - \vb^{-}(\rvx,t)} =\gamma \nabla_{\vx} \ln p(\rvx, t),
\end{equation}
\end{linenomath}
where $p(\rvx, t)$ solves the associated Fokker--Planck equation.
\end{lemma}
\begin{proof}
A variety of proofs can be found for this result \cite{nelson1967dynamical,anderson1982reverse,elliott1985reverse,follmer1984entropy,haussmann1985time}. Note the formulation of Equation \ref{eq:nelson} varies slightly across studies we use the one from \cite{anderson1982reverse} and \cite{nelson1967dynamical}.\footnote{See page 87, definition of osmotic velocity in \cite{nelson1967dynamical}.}
\end{proof}

\begin{algorithm} \label{alg:gipfp}
        \SetKwInOut{Input}{input}
        \Input{$\pi_0(\vx), \pi_1(\vy), \Q^{\gamma}_0$}
        $\Q^{*}_0 =\Q^{\gamma}_0$\\
        $i=0$ \\
            \Repeat{convergence}{
              $i := i + 1$ \\
             $ \P_i^{*} = \arginf_{\P \in \calD( \cdot, \pi_1)} \KL  (\P|| \Q^{*}_{i-1})$\\ 
             $ \Q_i^{*}  = \arginf_{\Q \in \calD(\pi_0, \cdot)} \KL  (\Q||  \P^{*}_{i})$
            }
        \Return{$\Q_i^{*}, \P_i^{*}$}
        \caption{Iterative Proportional Fitting (IPFP) (Cramer 2000), for the traditional SBP with Brownian prior set $\Q^{\gamma}_0:=\W^{\gamma}$. } % \cite{cramer2000probability}
\end{algorithm}
\subsubsection{Stochastic Control Formulation}\label{control}

Now that the FV constraint has been converted to an IV constraint we cast the problem into a stochastic control formulation to estimate the drift of each diffusion process. Following from the dynamic formulation, the control formulation casts the problem explicitly in terms of stochastic differential equations. The control formulation is used to enforce constraints as initial value problems. Furthermore, the drift based formulations of the SBP admit a reverse time formulation which starts the chain at the end of the interval and progresses the dynamics backwards in time to the start.
\begin{lemma}\cite{pavon1991free}
% Let $\Q_0^{\gamma}$ be the measure representing solutions to the SDE
% \begin{align*}
%     d\rvx^+(t)& = \vb_0^{+}(\rvx^+(t),t) dt + \sqrt{\gamma} \rvw^{+}(t), \;\;\; \rvx^+(0) \sim \pi_0^{ \Q_0^{\gamma}}
% \end{align*}
% and similarly for $\Q$,
Let the measure $\Q$ be defined by solutions to the SDE
\begin{linenomath}
\begin{equation*}
    d\rvx^+(t) = \rvb^{+}(t) dt + \sqrt{\gamma} d\rvw^{+}(t),\;\;\; \rvx^+(0) \sim \pi^{\Q}_0 
\end{equation*}
\end{linenomath}
Then the KL divergence $\KL\left(\Q \big|\big| \Q_0^{\gamma}\right)$ can be decomposed in terms of either the forward or reversed diffusion as
% \begin{linenomath}
\begin{align}\label{eq:free_energy_11}
     \KL\big(&\Q \big|\big| \Q_0^{\gamma}\big) = \KL(\pi^{\Q}_{\frac{1\mp 1}{2}} || \pi_{\frac{1\mp 1}{2}}^{ \Q_0^{\gamma}}) \\
     &+ \E_\Q\left[\int_0^1 \frac{1}{2\gamma}\big|\big|\rvb^{\pm}(t) - \vbb_0^{\pm}(\rvx^\pm(t),t)\big|\big|^2 dt\right],
\end{align}
% \end{linenomath}

% \begin{align}\label{eq:free_energy_1}
%      \KL\big(\Q \big|\big| \Q_0^{\gamma}\big) &= \KL(\pi^{\Q}_0 || \pi_0^{ \Q_0^{\gamma}}) + \E_\Q\left[\int_0^1 \frac{1}{2\gamma}\big|\big|\rvb^{+}(t) - \vbb_0^{+}(\rvx^+(t),t)\big|\big|^2 dt\right].\nonumber \\
%      &= \KL(\pi^{\Q}_1 || \pi_1^{ \Q_0^{\gamma}}) + \E_\Q\left[\int_0^1\frac{1}{2\gamma} \big|\big|\rvb^{-}(t) - \vbb_0^{-}(\rvx^-(t),t)\big|\big|^2 dt\right], \nonumber \\ 
%     %  d\rvx(t) &= \rvb^{+}(t) dt + \sqrt{\gamma} \rvw^+(t),
% \end{align}
%  Similarly $\KL\left(\Q \big|\big| \Q_0^{\gamma}\right)$ can also be decomposed in terms of reversed time diffusion:
% \begin{align}\label{eq:free_energy_2}
%      \KL\left(\Q \big|\big| \Q_0^{\gamma}\right) &= \KL(\pi^{\Q}_1 || \pi_1^{ \Q_0^{\gamma}}) \nonumber\\
%      &+ \E_\Q\left[\int_0^1\frac{1}{2\gamma} \big|\big|\rvb^{-}(t) - \vb_0^{-}(\rvx^-(t),t)\big|\big|^2 dt\right], \nonumber \\ 
%      d\rvx^-(t)& = \rvb^{-}(t) dt + \sqrt{\gamma} \rvw^{-}(t),
% \end{align}
where $\rvx^{-}\!(t)$ is the time reversal of $\rvx^+\!(t)$, and $\rvx^+\!(1)\!\sim\! \pi^{\Q}_1$.
\end{lemma}
\begin{proof}
This theorem follows by a direct application of the disintegration theorem (Appendix \ref{app:desint}) followed by Girsanov's theorem. A detailed proof can be found in \cite{pavon1991free}.
\end{proof}
% See proof in Appendix \ref{app:control}. 
Using the above decompositions, we can solve the SBP by minimising either decomposition in \eqref{eq:free_energy_11} over the space of random processes $\rvb^{\pm}(t)$ that satisfy a valid Itô SDE drift.
The backwards and forwards objectives are respectively: 
\begin{linenomath}
\begin{align} \label{eq:controlled_bridge_forward}
    \min_{\Q \in \calD(\pi_0, \pi_1)} \KL\big(\Q \big|\big|& \Q_0^{\gamma}\big) 
    \\=  \min_{\rvb^{\pm} \in \calB }  \E_\Q\Big[\int_0^1 \frac{1}{2\gamma}\big|\big|\rvb^{\pm}(t)  &- \vbb_0^{\pm}(\rvx^{\pm}(t),t)\big|\big|^2 dt\Big], \nonumber \\
    \text{s.t.}\;\;\; d\rvx^{\pm}(t) = \rvb^{\pm}(t) dt + \sqrt{\gamma}& \rvw^{\pm}(t),  \nonumber
    \\ \; \rvx^{+}\!(0) \!\sim\! \pi_{0}, \;\; \rvx^+\!(1) \!\sim\! \pi_1, \;\; \rvx^-\!(1)\!&\sim\!\pi_0, \;\; \rvx^-\!(0)\!\sim\!\pi_1. \nonumber %\nonumber% \\ 
    %  & \rvx^-(1) \sim \pi_1 , \;\; \rvx^-(0) \sim \pi_0.
\end{align}
\end{linenomath}
% \item Forward Objective:
% \begin{align} \label{eq:controlled_bridge_backward}
%     &\min_{\Q \in \calD(\pi_0, \pi_1)} \KL\big(\Q \big|\big| \Q_0^{\gamma}\big) = \nonumber\\
%     &\min_{\rvb^{-} \in \calB }  \E_\Q\left[\int_0^1 \frac{1}{2\gamma}\big|\big|\rvb^{-}(t)  - \vb_0^{-}(\rvx^-(t),t)\big|\big|^2 dt\right], \nonumber \\
% s.t.\;\;\; &d\rvx^-(t) = \rvb^{-}(t) t + \sqrt{\gamma} \rvw^{-}(t), \nonumber\\ 
%   & \rvx^-(1) \sim \pi_1 , \;\; \rvx^-(0) \sim \pi_0.
% \end{align}
% \end{itemize}

While we do not directly use the stochastic control formulations, the drift based formulation serves as inspiration for our iterative scheme. Specifically, the existence and parametrisation of an optimal drift as shown in \cite{pavon1991free,pavon2018data}. 
\begin{lemma}\label{lemma:pavon_opt}
\cite{pavon1991free} The objectives in Eq \eqref{eq:controlled_bridge_forward} 
% \begin{align} \label{eq:controlled_bridge_forward1}
%     &\min_{\rvb^{+} \in \calB }  \E_\Q\left[\int_0^1 \frac{1}{2\gamma}\big|\big|\rvb^{+}(t)  - \vb_0^{+}(\rvx(t),t)\big|\big|^2 dt\right], \nonumber \\
%     \text{s.t.}\;\;\; &d\rvx(t) = \rvb^{+}(t) dt + \sqrt{\gamma} \rvw^{+}(t), \nonumber \\ 
%     &\rvx^+(0) \sim \pi_0, \;\; \rvx^+(1) \sim \pi_1.
% \end{align}
have optimal drifts:
\begin{linenomath}
\begin{equation}
    \rvb^{+*}(t) \!= \gamma \nabla_{\vx}  \phi(\rvx^+(t), t),\;\rvb^{-*}(t) \!= \gamma \nabla_{\vx}  \hat\phi(\rvx^-\!(t), t)
    \label{eq:feedback_parametrisation}
\end{equation}
\end{linenomath}
where the potentials $\phi, \hat{\phi}$ solve the Schrödinger system \cite{pavon1991free}.
\end{lemma}
% Note that via Nelson's duality relation (Equation \ref{eq:nelson}) a direct corollary of Lemma \ref{lemma:pavon_opt} is that the optimal backwards drift takes the form $ \rvb^{-*}(t) =-\gamma \nabla_{\vx}  \phi(\rvx^-(t), t) + \gamma \nabla_{\vx} \ln p(\rvx^-(t), t)$. 
This Lemma is key in formulating our ML approach to IPFP since it justifies our parametrisation of the drift in terms of a deterministic function $\vb^{\pm*}$ i.e. $ \rvb^{\pm*}(t) = \vb^{\pm*}(\rvx^{\pm}(t), t)$. For a brief introduction to the Schrödinger system and potentials see Appendix \ref{appdx:system}.

\begin{figure} %{0.45\textwidth}
         \centering
         \includegraphics[width=0.7\columnwidth]{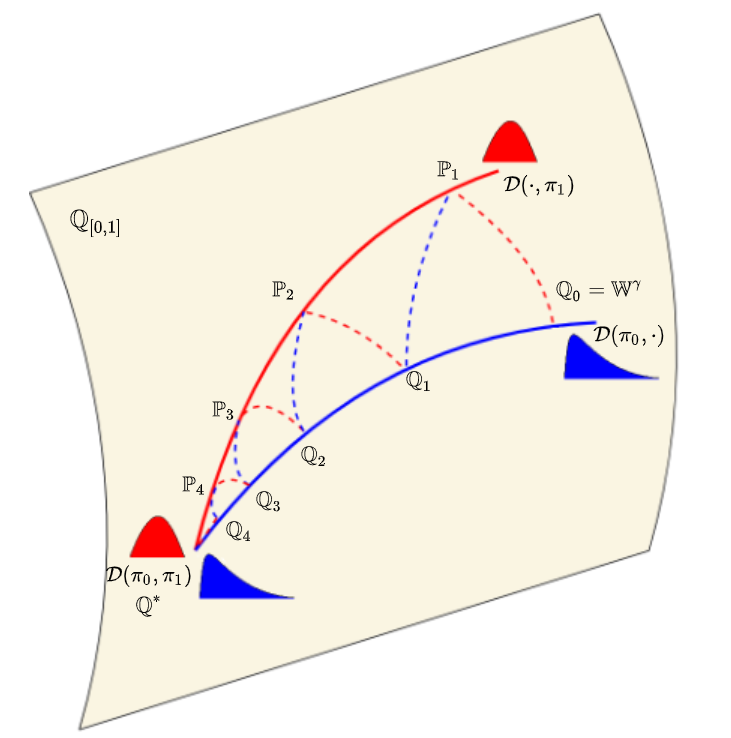}
         
    %  \end{minipage}
     \caption{ Illustration of the convergence of the IPFP Algorithm} \label{fig:IPFP}
\end{figure}
 %\begin{paracol}{2}
%\linenumbers
%\switchcolumn

\subsection{Iterative Proportional Fitting Procedure}\label{IPFP}

We now have two boundary value constrained diffusion processes solved through a stochastic control formulation. We use the iterative proportional fitting procedure (IPFP) to alternate between the forward and backward formulation with only one initial value constraint enforced at a time (see Figure \ref{fig:IPFP}), such that convergence to the full boundary value problem is guaranteed.

% We use an iterative proportional fitting procedure (IPFP) to find that consistent solution. \citep[see][]{cramer2000probability, bernton2019schr}
The measure theoretic version of IPFP we introduce is an extension of the continuous IPFP, initially proposed by \cite{kullback1968probability} to a more general setting over general probability measures (See \cite{cramer2000probability,bernton2019schr}). The convergence of IPFP has been shown in \cite{ruschendorf1995convergence} and further extensions and results have been presented in \cite{cramer2000probability,bernton2019schr}. The idea behind this family of approaches is to alternate minimising KL between the two marginal distribution constraints, 
\begin{linenomath}
\begin{equation*}
    \P_i^{*} \!= \!\arginf_{\P \in \calD( \cdot, \pi_1)} \KL  (\P|| \Q^{*}_{i-1}),\;\Q_i^{*} \! =\! \arginf_{\Q \in \calD(\pi_0, \cdot)} \KL  (\Q||  \P^{*}_{i}),
\end{equation*}
\end{linenomath}
% \begin{algorithm} \label{alg:gipfp}
% \SetKwInOut{Input}{input}
% \Input{$\pi_0(\vx), \pi_1(\vy), \Q^{\gamma}_0$}
% $\Q^{*}_0 =\Q^{\gamma}_0$\\
% $i=0$ \\
%     \Repeat{convergence}{
%       $i := i + 1$ \\
%      $ \P_i^{*} = \arginf_{\P \in \calD( \cdot, \pi_1)} \KL  (\P|| \Q^{*}_{i-1})$\\ 
%      $ \Q_i^{*}  = \arginf_{\Q \in \calD(\pi_0, \cdot)} \KL  (\Q||  \P^{*}_{i})$
%     }
% \Return{$\Q_i^{*}, \P_i^{*}$}
% \caption{Iterative Proportional Fitting (IPFP) \cite{cramer2000probability}, for the traditional SBP with Brownian prior set $\Q^{\gamma}_0:=\W^{\gamma}$. }
% \end{algorithm}
until convergence. The quantities ${\P}^{*}_i$ and  ${\Q}^{*}_i$ in Algorithm \ref{alg:gipfp} are known as half bridges \cite{pavon2018data}  and can be expressed in closed form in terms of known quantities,
% \begin{linenomath}
\begin{align}\label{eq:half_bridge_forward}
    {\P}^{*}_i\left( A_{[0,1)} \times A_1\right) &=  \int_{ A_{[0,1)} \times A_1}  \frac{d \pi_1}{ d p_1^{\Q^{*}_{i-1}}} d\Q^{*}_{i-1}, \\ 
    {\Q}^{*}_i\left(A_0 \times A_{(0,1]}\right) &=  \int_{A_0\times A_{(0,1]}} \frac{d \pi_0}{ dp_0^{\P^{*}_i}} d\P^{*}_i,
\end{align}
% \end{linenomath}
where $p_1^{\Q^{*}_{i-1}}$ and $p_0^{\P^{*}_i}$ are the marginals of $\Q^{*}_{i-1}$ and $\P^{*}_i$ at times $1$ and $0$ respectively. While a closed form expression to estimate half-bridges is known, its components are usually not available in closed form and require approximations. Note that the time reversed formulations \ref{control} and a simpler variant of Lemma \ref{lemma:pavon_opt} also apply to the half bridge problems. Applying the disintegration theorem (see Appendix \ref{app:desint}) we can reduce the IPFP introduced to the more popular instance of IPFP proposed by  \cite{kullback1968probability}. Furthermore for the case of discrete measures this algorithm reduces to the  Sinkhorn-Knopp algorithm \cite{sinkhorn1967concerning}.
% , by noting that $\Q^*_i(\cdot |\vx,\vy) =\P^*_i(\cdot |\vx,\vy) = \W(\cdot |\vx,\vy)$ remains invariant across iterations.
\section{Methodology}

In this section we introduce IPML by providing the theoretical foundations of our algorithm and convergence guarantees. Then, we propose a practical implementation of IPML based on a Bayesian non-parametric model (GP). These components allow us to solve the general SBP problem.

% To implement in practice, we need to introduce a probabilistic model for the drift terms parametrising the dual diffusion processes.
% The outline of this section is as follows:

% \begin{itemize}
%     \item First we will discuss how to rephrase the half bridge solution as a Maximum Likelihood Estimation problem for the drift. Furthermore we will provide convergence guarantees of our MLE rephrasing the true Half Bridge solution.
%     \item We will then discuss how this new formulation feeds back into the IPFP framework.
%     \item Finally we propose a practical method based on Gaussian Processes for estimating the drift.
% \end{itemize}

\subsection{Approximate Half Bridge Solving as Optimal Drift Estimation}

We present a novel approach to approximately solve the empirical Schrödinger bridge problem by exploiting the closed form expressions of the half bridge problem.  Rather than parametrising the measures in the half bridge and solving the optimisation numerically, we seek to directly approximate the measure that extremises the half bridge objective. We do this by using Gaussian processes \cite{williams2006gaussian} to estimate the drift of the trajectories sampled from the optimal half bridge measure.

We start from the following observation, which tells us how to sample from the optimal half bridge distribution (see  Appendix \ref{app:half} for proof).
\begin{observation} \label{obs:half}
We can parametrise a measure $\Q$ with its drift as  either of the SDEs,
\begin{linenomath}
\begin{equation*}
    d\rvx^{\pm}(t) =  \rvb^{\pm}(t) + \sqrt{\gamma} d \rvw^{\pm}(t) ,\;\;\rvx^{\pm}(0) \sim \pi_{0,1}^\Q. %,\:\rvx^-(0) \sim \pi_1^\Q.
\end{equation*}
\end{linenomath}
Then, we can sample from the solution to the half bridges,
\begin{linenomath}
\begin{equation*}
   \P^{*-}  \!\!= \!\!\arginf_{\P \in \calD(\cdot,\pi_1)} \!\!\KL(\P || \Q),\; \P^{*+}  \!\!= \!\!\arginf_{\P \in \calD(\pi_0,\cdot )} \!\!\KL(\P || \Q),
\end{equation*}
\end{linenomath}
via simulating trajectories (e.g.\ using the Euler-Maruyama (EM) method) following the SDEs
\begin{linenomath}
\begin{equation*}
     d\rvx^{\pm}(t) =   \rvb^{\pm}(t) + \sqrt{\gamma} d \rvw^{\pm}(t) , \;\;  \rvx^{\pm}(0) \sim \pi_{0,1}. %\\
    %  d\rvx^{+}(t) &=  \rvb^{+}(t) + \sqrt{\gamma} d \rvw^{+}(t) , \quad  \rvx(0) \sim \pi_0.
\end{equation*}
\end{linenomath}
Solutions to the above SDEs are distributed according to $\P^{*-}$ and $\P^{*+}$ respectively.
\end{observation}
% \begin{proof}See Appendix \ref{app:half}.
% \end{proof}
Intuitively, we are performing a cut-and-paste-styled operation by cutting the dynamics of the shortened (unconstrained) time interval and pasting the constraint to it at the corresponding boundary.
\begin{theorem}\label{theorem:mle-rev}
(Consistency of Reverse-MLE Formulations) Let $\left\{ \big(\vx^{(n)+}_{t_k} \big)_{k=0}^T \right\}_{n=0}^{N}$ be sampled/discretised trajectories from the SDE that represents the half bridge measure $\P^+$:
% \begin{linenomath}
\begin{equation*}
     d\rvx^{+}(t) =  \vb^{+}_0(\rvx^{+}(t),t) + \sqrt{\gamma} d \rvw^{+}(t) , \quad  \rvx(0) \sim \pi_0.
\end{equation*}
%\begin{align*}
%\end{align*}
% \end{linenomath}
Then carrying out maximum likelihood estimation of the time reversed drift ${\vb_0}^{-}$ on time reversed samples:
% \begin{linenomath}
\begin{align}\label{eq:mle}
   \prod_n p&\left( \big(\vx^{(n)+}_{t_k}\big)_{k=0}^T  \Big| \vb^{-}\right) \propto\\ \nonumber
    &\prod_n p(\rvx_{t_T}^{(n)+})\prod_{k=0}^T \gN \left( \vx^{(n)+}_{t_k- \Delta t} \Big| \vmu^{(n)+}_{t_k- \Delta t} , \gamma \Delta t\right),
\end{align}
where,
\begin{align}
    &\vmu^{(n)+}_{t_k- \Delta t} = \vx^{(n)+}_{t_k } -  \Delta t\vb^{-}\left(\vx^{(n)+}_{t_k}, 1-t_k\right) ,
\end{align}
% \end{linenomath}
% where,
% \begin{align}
%     &\vmu^{(n)+}_{t_k- \Delta t} = \vx^{(n)+}_{t_k } -  \Delta t\vb^{-}\left(\vx^{(n)+}_{t_k}, t_k - \Delta t\right) 
% \end{align}
and $\vb^{-}$ is the drift of our estimator SDE:
\begin{linenomath}
\begin{equation}
    d\rvx^{-}(t) =  \vb^{-}(\rvx^{-}(t), t) + \sqrt{\gamma} d \rvw^{-}(t) , \quad  \rvx^{-}(t)(0) \sim p,
\end{equation}
\end{linenomath}
converges in probability to the to the true dual drift $\vb^{-}_0(\vx,t)=\vb^{+}_0(\vx,t) -\gamma \nabla_{\vx} \ln p(\vx, t)$ where $\vb^{+}_0(x, t)$ is the optimal half bridge drift for $\P^{+}$ as $N \rightarrow \infty, \Delta t \rightarrow 0$ (see Appendix \ref{app:cons} for proof).
\end{theorem}
Note that w.l.o.g. the above result also holds for estimating the forward drift from backward samples. The combination of Observation \ref{obs:half} and Theorem \ref{theorem:mle-rev} constitute one the main contributions of this work as they allow us to solve half bridges with a simple regression objective.  It is important to highlight that in practice for most interesting SDEs we can only sample approximately using consistent schemes, for these settings our proof of Theorem \ref{theorem:mle-rev} does not hold, however we believe it is possible to adapt the result to this case.
% \begin{proof}
% See Appendix \ref{app:cons}.
% \end{proof}
% \begin{wrapfigure}{l}{0.5\textwidth}
% \begin{minipage}{0.5\textwidth}
% \begin{algorithm}[H] \label{alg:gp_gipfp}
% \SetKwInOut{Input}{input}
% \Input{$\pi_0(\vx), \pi_1(\vy), \Q_0^{\gamma}$}
% Initialise: $i := 0$  $\Q^{*}_0 := \Q^{\gamma}_0$\\
% % Obtain or estimate backward drift of prior: \\
% $\bar{\vb}^{-}_{\Q_{0}}(\vx, t) := \text{ObtainBackwardDrift}(\Q^{*}_0)$\\

% \Repeat{convergence}{
%       $i := i + 1$ \\
%       $\left\{ \!\vx^{(m)-}_{[1:T]}\!\right\}_{m} \!\!\!\!:= \! \!\text{SDESolve}\Big( \!\!-\bar{\vb}^{-}_{\Q_{i\text{-}1}},\!\gamma, \pi_1, \Delta t\!\Big)$\\
%      $ \bar{\vb}^{+}_{\P_{i}}\!\!:= \!\text{DriftFit}\Big(\!  \text{Reverse}\!\big( \!\left\{\vx^{(m)-}_{[1:T]} \right\}_{\!m}\! \big),  \frac{\gamma}{\Delta t }  \!\Big)$\\ 
%      $\left\{ \vx^{(n)+}_{[1:T]} \right\}_{n}  \!\!\!:= \text{SDESolve}\Big(\bar{\vb}^{+}_{\P_{i}}, \gamma, \pi_0, \Delta t\Big)$ \\
%      $\bar{\vb}^{-}_{\Q_{i}}\!\!\!:=\text{DriftFit}\Big(\!\text{Reverse}\big( \left\{ \!\vx^{(n)+}_{[1:T]}  \right\}_{\!n} \!\big), \frac{\gamma}{\Delta t } \!\Big)$
%     }
% \Return{$\bar{\vb}^{-}_{\Q_{i}}, \bar{\vb}^{+}_{\P_{i}}$}
% \caption{IPML}
% \end{algorithm}
% \end{minipage}
% \end{wrapfigure}

\subsection{On the Need for Time Reversal}\label{need_time_rev}
We have already provided the intuition as to how time reversal enables the exchange of initial value constraints for final value constraints. In what follows we give a more detailed technical review of this important simplification.
 % At first it may seem unnecessary as to why we consider the reverse formulation of the half bridges. We will now present a short motivation of why this is a useful parametrisation.
Consider the first half bridge problem for a given IPFP iteration, $\P^{*-} = \arginf_{\P \in \calD(\cdot,\pi_1)}\KL(\P || \Q)$.
When formulated directly in terms of the forward drift and the forward diffusion it reduces to a stochastic control problem which is subject to dynamics,
\begin{linenomath}
\begin{equation*}
   d\rvx^{+}(t) =   \rvb^{+}(t) + \sqrt{\gamma} d \rvw^{+}(t) , \quad  \rvx(1) \sim \pi_1,
\end{equation*}
\end{linenomath}
which is a forward SDE with a terminal hitting condition. Unlike a forward SDE with an initial value problem the above SDE is not trivial to sample from, one approach would be to solve the backwards Kolmogorov equation subject to the terminal condition. However, this approach requires (1) solving a parabolic PDE that typically involves mesh based methods that do not scale well in high dimensions and (2) carrying out density estimation on the samples from $\pi_1$, also problematic in high dimensions. However, if we consider the time reversed process, the terminal condition becomes an initial value problem, $d\rvx^{-}(t) =   \rvb^{-}(t) + \sqrt{\gamma} d \rvw^{-}(t) , \quad  \rvx(0) \sim \pi_1,$ for which it is easy to sample consistent trajectories via the Euler-Maruyama (EM) discretisation without mesh based methods or additional density estimations. 
%\end{paracol}

\begin{figure}
     \hspace*{0.13cm}
         \centering
          \includegraphics[width=0.85\columnwidth]{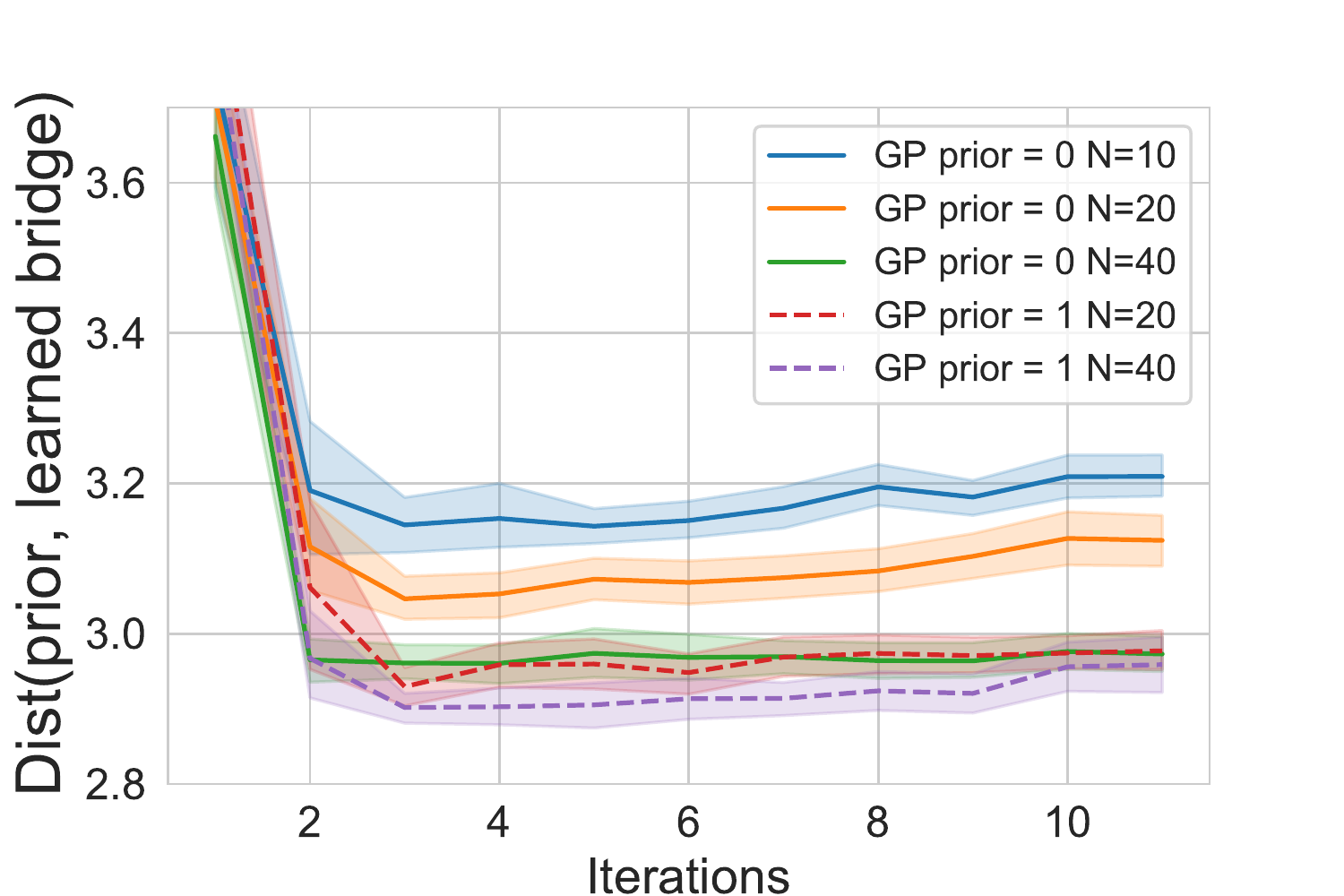}
         \caption{Distance between the prior and the learned bridge as a function of the iterations, number of samples used (N) and whether the prior is used for the GP ($T\!=\!\frac{1}{\Delta t}\!=\!100$).}
         \label{fig:compare_param}
 \end{figure}
 %\begin{paracol}{2}
%\linenumbers
%\switchcolumn
%Forwards and backwards diffusion of learned SBP between unimodal and multimodal boundary distributions (see section \ref{sec:simple} for experimental details).
\subsection{Iterative Proportional Maximum Likelihood (IPML)}

Combining the KL minimisation routines in the original IPFP algorithm with our MLE based drift estimation gives Algorithm \ref{alg:gp_gipfp}. The routine $\text{DriftFit}$ fits a dual drift on the sampled trajectories, and can be parametrised by any function estimation procedure with consistency guarantees. The $\text{SDESolve}$ routine generates $M$ trajectories using the EM method. 

%At each iteration, the computational cost of the algorithm is O(n
\textbf{Drift Estimation with Gaussian Processes}: One can choose any parametric or non parametric model to carry out the $\text{DriftFit}$ routine. In this section we will briefly introduce our implementation. Our $\text{DriftFit}$ routine is based on the work in \citet{papaspiliopoulos2012nonparametric,ruttor2013approximate} that uses GPs to estimate the drift of SDEs from observations. We refer to this routine as $\text{GPDriftFit}$. We can restate the regression problems given by Equation \ref{eq:mle} and its reversed counterpart in the following form:
\begin{align*}
        \frac{\vx^{(n)-}_{t_k} \!\!\!\!\!-\! \vx^{(n)-}_{t_k - \Delta t}}{\Delta t} \!&= \!\vb^{+}\!\left(\!\vx^{(n)-}_{(t_k \!- \Delta t)}, 1-(t_k \!-\! \Delta t)\right) \!+\!\frac{\sqrt{\bm{\gamma}} }{\Delta t}\epsilon ,\\ \frac{\vx^{(m)+}_{t_l \!\!- \Delta t} \!\!- \vx^{(m)+\!\!}_{t_l}}{\Delta t} \! &= \!-\vb^{-}\left(\!\vx^{(m)+}_{t_l}, 1-t_l \right) \!+\!\frac{\sqrt{\bm{\gamma}} }{\Delta t}\epsilon,
\end{align*}
where $\epsilon \sim \calN(\vzero, \I)$. Placing GP priors on the drift functions $\vb^{-} \sim \mathcal{GP} $ and $\vb^{+} \sim \mathcal{GP} $, we arrive at standard multioutput GP formulation.
Following \citet{ruttor2013approximate,batz2018approximate}, we assume the dimensions of the drift function are independent \footnote{Assuming that the drift dimensions are decoupled is the least restrictive assumption as coupling them would impose a form of regularization.} (equivalent to imposing a block-diagonal kernel matrix on a multi-output GP) and thus we fit a separate GP for each dimension. See Appendix \ref{appdx:predmean} for the specification of the predictive means.

Note that the advantage of the Bayesian GP interpretation of this procedure is that the GP posterior under certain conditions is naturally conjugate to the posterior of the SDE drift when modelled with a GP prior \cite{pokern2013posterior,papaspiliopoulos2012nonparametric}. Furthermore, the Bayesian non-parametric formulation allows for the encoding of the prior drift function as the mean function of the GP prior (i.e. $\vb \!\sim \!\mathcal{G}\mathcal{P}(  \vbb_0^{\Q^{\gamma}_0} , k)$). This becomes helpful when trying to sample unlikely paths and evaluating the fitted drift at samples that were not observed by the GP during training, here we want the fitted drift to fall back to the prior rather than to a Brownian motion given by a GP prior with 0 mean.

We now have the relevant ingredients to carry out IPML as specified in Algorithm  \ref{alg:gp_gipfp}. The computational cost of IPML with a Gaussian Process is detailed in Appendix \ref{app:comp}. The majority of the computation is spent fitting the GP in DriftFit at each iteration and scales as a function of the time discretization used as well as the size of $\pi_0$ and $\pi_1$.

\begin{algorithm} \label{alg:gp_gipfp}
\SetKwInOut{Input}{input}
\Input{$\pi_0(\vx), \pi_1(\vy), \Q_0^{\gamma}$}
Initialise: $i := 0$  $\Q^{*}_0 := \Q^{\gamma}_0$\\
% Obtain or estimate backward drift of prior: \\
$\bar{\vb}^{-}_{\Q_{0}}(\vx, t) := \text{ObtainBackwardDrift}(\Q^{*}_0)$\\

\Repeat{convergence}{
      $i := i + 1$ \\
      $\left\{ \!\vx^{(m)-}_{[1:T]}\!\right\}_{m} \!\!\!\!:= \! \!\text{SDESolve}\Big( \!\!-\bar{\vb}^{-}_{\Q_{i\text{-}1}},\!\gamma, \pi_1, \Delta t\!\Big)$\\
     $ \bar{\vb}^{+}_{\P_{i}}\!\!:= \!\text{DriftFit}\Big(\!  \text{Reverse}\!\big( \!\left\{\vx^{(m)-}_{[1:T]} \right\}_{\!m}\! \big),  \frac{\gamma}{\Delta t }  \!\Big)$\\ 
     $\left\{ \vx^{(n)+}_{[1:T]} \right\}_{n}  \!\!\!:= \text{SDESolve}\Big(\bar{\vb}^{+}_{\P_{i}}, \gamma, \pi_0, \Delta t\Big)$ \\
     $\bar{\vb}^{-}_{\Q_{i}}\!\!\!:=\text{DriftFit}\Big(\!\text{Reverse}\big( \left\{ \!\vx^{(n)+}_{[1:T]}  \right\}_{\!n} \!\big), \frac{\gamma}{\Delta t } \!\Big)$
    }
\Return{$\bar{\vb}^{-}_{\Q_{i}}, \bar{\vb}^{+}_{\P_{i}}$}
\caption{IPML}

\end{algorithm}

% \subsubsection{ Deterministic Annealing -- $\gamma$-scaling}

% % Comment: More popularly known as $\epsilon$-scaling. TODO: Add references , justify method, explain how / why its used to speed up convergence, mention intuitive analogy of exploration vs exploitation.

% The $\gamma$-scaling procedure more popularly known as $\epsilon$-scaling is a form of deterministic annealing \cite{kirkpatrick1983optimization,bertsekas1979distributed} applied to IPFP/Sinkhorn based algorithms that models $\gamma^{-1}$ as a temperature and cools it down every few iterations of the IPFP procedure until reaching a desired temperature. While practical success of $\gamma$-scaling for solving entropy regularised optimal transport problems has been demonstrated to a great extent in prior works \cite{schmitzer2019stabilized,feydy2020geometric,feydy2019fast}, theoretical analysis is so far lacking, however \cite{schmitzer2019stabilized} pioneers in this front and presents some of the first stability and potential convergence guarantees for $\gamma$-scaling.

% This procedure is known to accelerate IPFP based algorithms dramatically in practice in scenarios where $\gamma$ is small causing IPFP to converge very slowly and in some cases diverge \cite{feydy2020geometric}.
\section{Related Methodology}

In this section we carry out a conceptual comparison with two pre-existing numerical approaches for solving the static SBP. While our goal is to solve the dynamic SBP, the solution of the static SBP can be used to construct that of the dynamic SBP \cite{pavon2018data}, and this connection is central to our discussion. We would also like to highlight that an algorithm akin to IPML has been proposed concurrently and independently by \cite{de2021diffusion}, the main difference with our algorithm is that they estimate the drifts of the SDEs using neural networks score matching while we use using Gaussian processes and MLE based ideas.

\paragraph{Sinkhorn-Knop  Algorithms:} Within the machine learning community the static SBP with a Brownian motion prior is popularly known as entropic optimal transport \cite{cuturi2013sinkhorn}. In this formulation there are no trajectories and the empirical distributions are treated as discrete measures $[\bm{\pi_{0}}]_{j}\!\!=\!\frac{1}{N}$ and $[\bm{\pi_{1}}]_{i}\!\!=\! \frac{1}{M}$ thus the objective is given by $\min_{Q \in \calD(\bm{\pi_0}, \bm{\pi_1})} \langle Q, C^{\Q_0^\gamma}  \rangle + \gamma h(Q)$, where $h$ is the discrete entropy, $\langle \cdot,\cdot\rangle$ computes the dot product between two matrices. The cost matrix $C$ corresponds to the log transition density induced by the prior SDE:
%\begin{align*} 
    $C^{\Q_0^\gamma}_{ij} = \ln p^{\Q_0^{\gamma}}(\vy_i | \vx_j)$
%\end{align*}
which in the case of a Brownian motion prior reduces to a scaled Euclidean distance.

Once the problem has been discretized  as described, the Sinkhorn-Knopp algorithm \cite{sinkhorn1967concerning} can be applied directly to fit an optimal discrete transport map (and discrete SBP potentials $\phi, \hat{\phi}$) between the two distributions. Recent work \cite{chizat2020faster} has showed empirical success in forming a continuous approximation of the SBP potentials using the logsumexp formula\cite{chizat2020faster}, however, it still remains to formally analyse the accuracy of the logsumexp potentials.

Estimating the cost matrix $C^{\Q_0^\gamma}_{ij} $ required by the Sinkhorn algorithm requires a mixture of both density estimation and further simulation of the prior. Furthermore, once we have the logsumexp potentials, additional high dimensional integrals must be estimated every time we wish to evaluate the optimal drift. Details are discussed in Appendix \ref{app:sink}. In addition, the Sinkhorn-Knopp algorithm still faces challenges in high dimensional spaces, specially for small values of $\gamma$ \cite{feydy2020geometric}. Many proposed enhancements and literature \cite{cuturi2013sinkhorn,feydy2019interpolating} have focused on cost functions that implicitly require a Brownian motion prior and thus do not apply to our general setting. 

% In short we can conclude that adapting Sinkhorn-Knopp based algorithms to estimate the solutions to the dynamical SBP merit its own research and algorithm development which is not the focus of this work. 
% \begin{align}
%     \int \phi_1&(\vy; {\beta} )   \hphi_1^{(i)}(\vy) d\vy =\nonumber\\
%     &\int  \hphi_0(\vx; \hat{\beta} ) \int p^{\Q_0^{\gamma}}(\vy| \vx) \phi_1^{(i)}(\vx; \beta) d\vy d\vx.\label{eq:hdint2}.
% \end{align}
\paragraph{Data Driven Schrödinger Bridge (DDSB):} The method proposed by \cite{pavon2018data} is perhaps the most similar approach to our approach and consists of iterating  two coupled density estimation objectives (see Appendix \ref{app:pavonap}) fitted at the marginals until convergence. While conceptually similar to our approach, there are 3 key differences. As with Sinkhorn-Knopp based methods, their approach aims to solve the static SBP. Once converged it requires further approximation to estimate the optimal drift. The coupled maximum likelihood formulation of the static half bridges in \cite{pavon2018data} is based on un-normalized density estimation with respect to the SBP potentials $\phi, \hat{\phi}$. The coupling of $\phi, \hat{\phi}$ in these objective does not directly admit the application of modern methods in density estimation since it does not allow us to freely parametrise MLE estimators for the boundaries thus neural density estimators such as \cite{kingma2013auto,papamakarios2017masked,papamakarios2019neural} cannot be taken advantage of to circumvent the computation of the partition function. Regression problems are ubiquitous in machine learning (ML) and thus why we believe that this formulation can be very impactful as it allows us to leverage all these methods from ML. Experimentally, regression methods have been observed to scale better to high dimensional problems than density estimation methods, we observed similar evidence as we were unable to scale up the DDSB method beyond 2 dimensions. 

Additionally, to compute the normalizing term in the DDSB objective, we have to estimate a multidimensional integral that is not taken with respect to a probability distribution. This poses a difficult challenge in high dimensions. For a more detailed commentary, see Appendix \ref{app:pavonap}. Note that the method by \cite{pavon1991free} uses importance sampling to estimate these quantities which performs poorly beyond 2 dimensions.

It worth noting that we are interested in a method that can obtain a dynamic interpolation between two distributions. A downside of solving the static bridge either by Sinkhorn or \cite{pavon2018data} is that it does not directly provide us with an estimate of the optimal dynamics. In order to obtain an estimate of the optimal drift we require a series of approximations to estimate the integrals in \ref{eq:potentialt}, \ref{eq:potentialt2}, thus every time we evaluate the optimal drift using these approaches we have to simulate the prior SDE $\gO(N+M)$ times. This makes the run-time of obtaining the drift and dynamical interpolation expensive, since for each Euler step we take we have to simulate another SDE and backpropagate through it to evaluate the drift.  

Finally we would like to highlight that a series of modern approaches to generative modelling \cite{wang2021deep,de2021diffusion,huang2021schrodinger,kingma2021variational} motivate both empirically and theoretically the gain in accuracy obtained in generative modelling tasks when using a dynamical approach rather than a static one.

% In conclusion the method by \cite{pavon2018data} does not naturally lend itself to estimating the optimal dynamic bridge and requires the same approximations as with Sinkhorn-Knopp based algorithms.  Furthermore the normalized likelihood setup in this approach requires the estimation of non trivial high-dimensional integrals and does not scale well to high dimensions.
\section{Numerical Experiments}
In this section, we demonstrate the capability of IPML to solve the Schrödinger bridge while efficiently incorporating priors on a range of different tasks from synthetic experiments to embryo cells.\footnote{Code supporting experiments can be found at \url{https://github.com/AforAnonyMeta/IPML-2548}.}

%\footnote{Code supporting experiments can be found at \url{https://github.com/franciscovargas/GP_Sinkhorn}.}
% https://github.com/franciscovargas/GP_Sinkhorn

% \subsection{}
\label{sec:simple}
\subsection{Simple 1D and 2D Distribution Alignment}
%REMOVED: In the unimodal experiment, the marginals are distributed as follows: $\pi_0 \sim \mathcal{N}(0,1),\pi_1 \sim \mathcal{N}(4,0.1)$
The first experiment considered is a simple alignment experiment where $\pi_0$ and $\pi_1$ are either unimodal or bimodal Gaussian distributions (see Appendix \ref{sec:1d} for exact details on the distributions). In Table \ref{tab:pavon_compare} we compare the accuracy of the fitted marginals with our implementation of the Data-Driven Schrödinger Bridge (DDSB) by \cite{pavon2018data}. The scoring metrics used are the Earth mover's distance (EMD) as well as a Kolmogorov-Smirnov (KS) statistic on both sample sets.
% %\end{paracol}
\begin{table}
% \renewcommand{\arraystretch}{1.2} % Default value: 1
% \begin{table}[]
% \label{tab:pavon_compare}
% \begin{minipage}[b]{0.2\textwidth}
\caption{Performance comparison of the fitted marginals between DDSB and IPML on unimodal and bimodal experiment. }
\centering
\begin{adjustbox}{width=\columnwidth}
\begin{tabular}[t]{ccccccccc}
\toprule
      & \multicolumn{4}{c}{Unimodal}                              & \multicolumn{4}{c}{Bimodal}                               \\
      & \multicolumn{2}{c}{$\pi_0$} & \multicolumn{2}{c}{$\pi_1$} & \multicolumn{2}{c}{$\pi_0$} & \multicolumn{2}{c}{$\pi_1$} \\
      & KS            & EMD         & KS           & EMD          & KS           & EMD          & KS            & EMD         \\ \hline
DDSB & 0.17         & 0.34        & 0.19         & 0.13         & 0.18         & 0.12         & 0.07          &   0.04          \\
IPML  & 0.06         & 0.10         & 0.13         & 0.04         & 0.05         & 0.04         & 0.07          &    0.15         \\ \bottomrule
\end{tabular}
\end{adjustbox}
\label{tab:pavon_compare}
\end{table}
We were unable to get DDSB to work well when $\pi_0$ and $\pi_1$ were distant from each other. In these distant settings, DDSB collapses the mass of the marginals to a single data point (see Appendix \ref{appsec:delta}), as a result for the unimodal experiment we had to set $\gamma=100$ for the DDSB approach to yield sensible results. We can observe IPML obtains better marginals overall and at a lower value of $\gamma=1$.

We carried out 2D experiments with our approach to show the ability of our method to diffuse from a simple uni-modal distribution to a multi-modal distribution as in Figure \ref{fig:well_and_forward} where our learned bridge successfully splits. Furthermore, we can visually observe how the forward and backwards trajectories are mirror images of each other as expected.% Unfortunately, we were unable to get the method by \cite{pavon2018data} to give sensible results on these 2D examples, thus pursuing a comparison with their approach in 1D.
% \subsection{2D Double Well Experiments}

\label{sec:double}
\subsection{2D Double Well Experiments} In the double well experiment, we illustrate how to incorporate an arbitrary functional prior and learn the distribution over paths connecting $\pi_0$ and $\pi_1$. In order to encode prior information, we experiment with the potential well illustrated in Figure \ref{fig:well_tr} and \ref{fig:well} (Figure \ref{fig:well} can be found in the Appendix). The boundary distributions $\pi_1, \pi_0$ are taken to be Gaussian distributions centred at the centre of each well respectively (See Appendix \ref{sec:well} for the experiment specification and IPML parameters). 
%\end{paracol}

 \begin{figure*}[!t]
    %\widefigure
    \centering
    \includegraphics[scale=0.54]{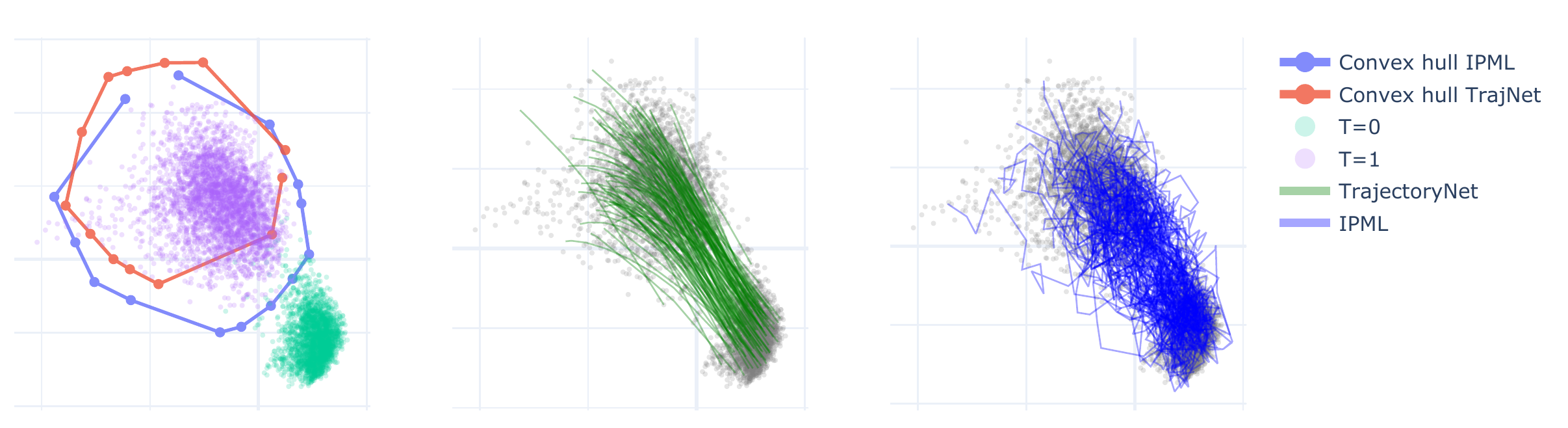}
    \caption{The left plot represent the first and last frame of the cells data and the convex hull of the forward model for TrajectoryNet and IPML. The right plots represent sample trajectories for Trajectorynet and IPML.}
    \label{fig:single_cel_trajectories}
\end{figure*}

 %\begin{paracol}{2}
%\linenumbers
%\switchcolumn

The motivation behind this experiment is to show that the SBP with this prior follows low energy (according to the well's potential function) trajectories for configurations of particles sampled at the wells. Intuitively, we can expect the learned trajectories to avoid the high energy peak located at $\vx=(0,0)$ and go via the ``passes'' on either side. Note that if we estimated the optimal transport (OT) geodesics between $\pi_0$ and $\pi_1$ or similarly ran IPML with a Brownian motion prior, the learned optimal trajectories would go right through the middle, which is the highest energy path between wells.

The prior is incorporated in the algorithm in two different ways, (1) by having the first drift $\Q^{*}_0 $ to follow the negative derivative of the potential function $d \rvx(t) = -\nabla_{\vx}U (\rvx) + \gamma d\rvw(t)$ and (2) by setting the mean function of the GP used to fit the drift as the negative derivative of the potential function. As illustrated in Figure \ref{fig:well_tr}, the learned trajectories between the two wells avoid the main high energy region and go through via lower energy passes, thus respecting the potential prior. However, the behavior of the trajectory will differ depending on the choice of kernel, as illustrated in Figure \ref{fig:well} (Appendix), underlying the need for its careful consideration.  We compared our approach to DDSB using the mean squared error distance to the prior in our evaluation. The results were \textbf{DDSB: 22.1, IPML (no prior): 19.9, IPML: (prior) 18.7}. As we can see IPML considerably outperforms DDSB.

\subsection{Finite sample/iteration convergence} 
Theorem \ref{theorem:mle-rev} provides us with asymptotic guarantees. However, it does not extend to the finite sample and discretisation case. To highlight the importance of finite effects on IPML, we carried out this analysis empirically. In Figure \ref{fig:compare_param} we plot an empirical estimate of the error term in the control formulation of the SB (Eq. 5). This term is effectively the mean squared error between the learned drift and the prior drift (gradient field of the well). We analyze this metric for different values of $N$ (number of samples) and $\Delta_t$ (discretization factor). We observe that IPML quickly reaches a low error valley, then, the cumulative error from the successive finite sample MLE can be observed and the drift starts to slowly deviate from the prior, this motivates early stopping. As $N$ is increased, IPML achieves lower error faster and deviates less from the prior in later iterations. Finally, we observe placing the drift prior via the GP has a significant effect in improving the error and its convergence. Additionally in Appendix \ref{appdx:finite} we detail how this question could be approached from a theoretical perspective while underlining its significance and difficulty as illustrated by the lack of such analyses in related algorithms \cite{pavon2018data,bernton2019schr}.

\label{sec:cells}
\subsection{Single Cell - Embryo Body (EB) Data set}
We perform an experiment on an embryoid body scRNA-seq time course \cite{tong2020trajectorynet}. Single-cell RNA sequencing enables accurate identification of cells at specific time-points, however, all cells are destroyed by measurement. This prevents modelling single-cell trajectory and instead we rely on modelling the data-manifold at discrete time points. The datasets consist of 5 time points illustrated in Figure \ref{fig:cell_volume}. To evaluate the performance of the algorithms, we fit the models at the endpoints ($T=1,5$) and predict the intermediate frames. The metric used is the Earth mover's distance between the data at intermediate frames and the predicted distribution. We evaluate the performance at the endpoints by considering the prediction of the forward model at $T=5$ and the backward model at $T=1$.

\begin{table}
% \renewcommand{\arraystretch}{1.2} % Default value: 1
% \begin{table}[]
% \label{tab:pavon_compare}
% \begin{minipage}[b]{0.2\textwidth}
\caption{Earth Moving Distance (EMD) on the EB data. EXP stands for the exponential kernel and EQ for exponentiated quadratic. The column "full" represents EMD averaged over all frames whereas "path" is averaged over the intermediate frames ($T={2,3,4}$).}
\centering
\begin{adjustbox}{width=\columnwidth}
\begin{tabular}[t]{@{}llllllll@{}}
\toprule
\multirow{2}{*}{} & \multirow{2}{*}{T=1} & \multirow{2}{*}{T=2} & \multirow{2}{*}{T=3} & \multirow{2}{*}{T=4} & \multirow{2}{*}{T=5} & \multicolumn{2}{c}{Mean} \\
                  &                      &                      &                      &                      &                      & Path        & Full       \\ \midrule
TrajectoryNet     & 0.62                 & 1.15                 & 1.49                 & 1.26                 & 0.99                 & 1.30        & 1.18       \\
IPML EQ                & 0.38                 & 1.19                 & 1.44                 & 1.04                 & 0.48                 & 1.22        & 1.02        \\
IPML EXP & 0.34 & 1.13 & 1.35 & 1.01 & 0.49 & 1.16 & 0.97\\
OT                & Na                   & 1.13                 & 1.10                 & 1.11                 & Na                   & 1.11        & Na          \\
\bottomrule
\end{tabular}
\end{adjustbox}
\label{tab:cell_tab}
\end{table}

We compare the performance with two methods. The first one, TrajectoryNet \cite{tong2020trajectorynet}, uses continuous normalizing flows with a soft constraint based on optimal transport. The second, leverages the McCann interpolant \cite{mccann2011five,schiebinger2019optimal} to interpolate the discrete OT solution; this corresponds to the linear interpolation induced by the transport map.

The results are summarized in Table \ref{tab:cell_tab}. In most frames, IPML outperforms TrajectoryNet and performs similarly to OT. As the noise (volatility) in the trajectories goes to 0, IPML theoretically converges to the OT solution with linear geodesics. So without any additional prior information, we would not expect IPML to outperform OT. However, when dealing with finite data, the DriftFit procedure allows for some non-linearity in the trajectories if it improves the fit. As a result, we do see some differences in the convex hull displayed in Figure \ref{fig:single_cel_trajectories} where we observe a better coverage of the single cell observations. We hypothesize this explains the improvement in performance vs OT for frame 4. The performance of IPML may be further improved by incorporating domain-specific knowledge as a prior. It would outperform OT in those settings. 

%With a suitable prior, IPML results should outperform OT across all frames, however due to lacking domain specific knowledge of the problem we were unable to construct such a prior.
%\end{paracol}
\begin{figure}[t!]
    %\widefigure
    \centering
    \hspace{1cm}
    \begin{minipage}{0.28\textwidth}
    \centering
    \includegraphics[width=\textwidth]{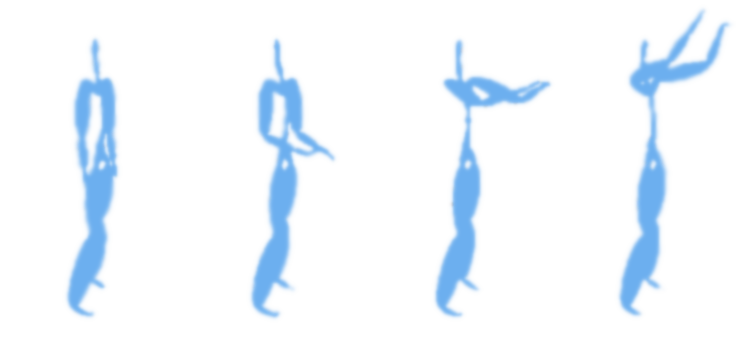}
    % \caption{Learned SBP trajectories in the double well experiment of Section \ref{sec:double}, with prior $\Q_0^\gamma$ expressed in terms of an energy landscape $U(x,y)$ as $d \rvx(t) = -\nabla_{\vx}U (\rvx) + \gamma d\rvw(t)$}
    % \caption{}
    \end{minipage}
    \hspace*{\fill}
    \begin{minipage}{\columnwidth}
    \centering
    \includegraphics[width=\columnwidth]{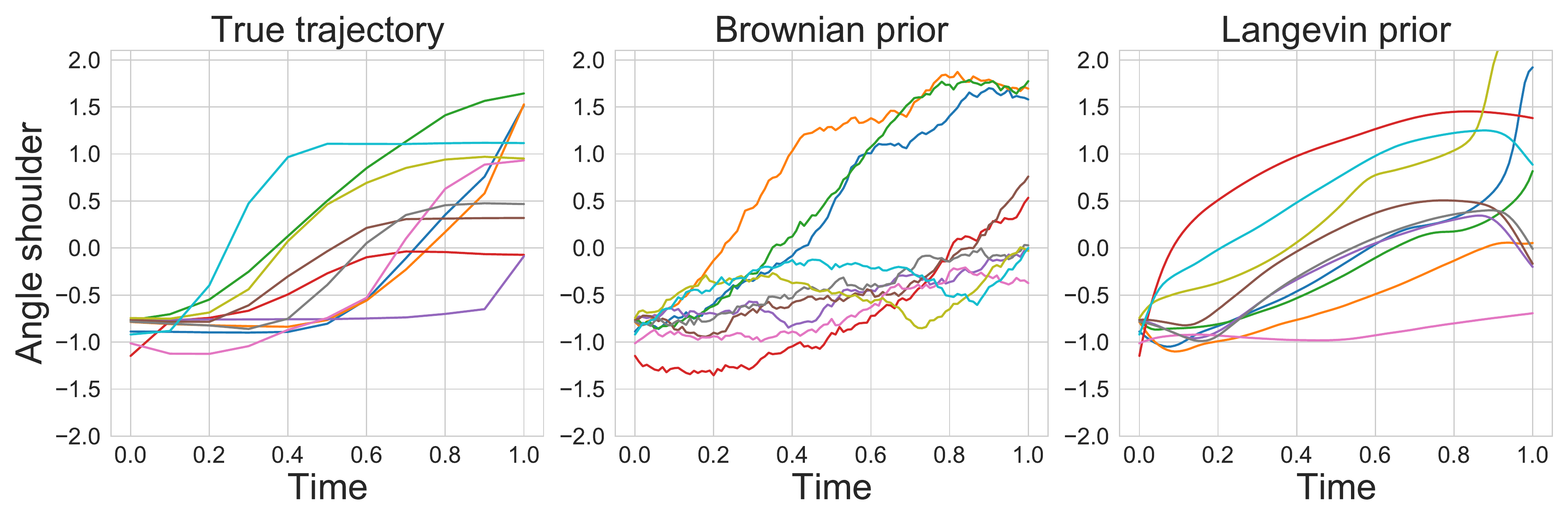}
    % \caption{Fitted SBP on the start and end data slices with Brownian prior on single cell human Embryo data. Observations depicted as red point clouds. See section \ref{sec:cells} for experimental details.}
    % \caption{}
    \end{minipage}
    \caption{\emph{Top}: 3D animation of the basketball signal motion modeled in Section \ref{subsq:mocap} (motion from left to right). \emph{Bottom}: Trajectories of the shoulder's oriented angle sensor through the motion. The two plots on the right demonstrate IPML's fit using a Brownian and Langevin prior.}
    \label{fig:motion_plot_main}
\end{figure}
 %\begin{paracol}{2}
%\linenumbers
%\switchcolumn
% \subsection{Motion Capture}
\label{subsq:mocap}
\subsection{Motion Capture}
In this experiment, we demonstrate how IPML can be used to model human motion from sensor data.\footnote{Data from The CMU Graphics Lab Motion Capture Database 
funded by the NSF (\url{http://mocap.cs.cmu.edu}).} The motion corresponds to a basketball movement where the subject raises both arms simultaneously as illustrated in Figure \ref{fig:motion_plot_main} and each sensor corresponds to an oriented angle. We focus on modelling the right shoulder and elbow where the starting distribution corresponds to the leftmost image and the ending one the rightmost. This results in a 4 dimensional space as we model both position and velocity for each sensor. We compare the fit of IPML using a Brownian and a 2nd order linear ODE (Langevin) prior. The experimental details can be found in \ref{app:mot}. As illustrated in figure \ref{fig:motion_plot_main}, IPML is able to approximately model the dynamics using both priors. The Brownian prior displays noisy trajectories as expected in contrast to the Langevin prior that, by construction, smooths out the predicted positions (due to the 2nd order term). We observe that the Langevin prior approximates the step function nature of the true trajectory more closely than the Brownian prior, additionally, we can see that it also produces slightly better alignments.

\section{Limitations and Opportunities}

In this work, we propose to use GPs to estimate the drift, however, IPML could also be used with a parametric function estimator (e.g. a neural network). This could be useful with high-dimensional data where GPs may underperform. The main advantage of using a GP is its capacity to incorporate functional priors via the mean function. This can be useful in applications such as molecular dynamics where a potential function may be available. In contrast, implementing functional regularization in a neural network would require approximating a non-trivial high-dimensional integral to estimate the mean squared error between the parametric estimator and the functional prior.

A particular useful extension would be to adapt the SBP to work with more general forms of volatility functions that are not constant. This can be used, for example, in enforcing positivity constraints on a stochastic process via a geometric Brownian motion prior; this has applications in modelling biological signals such as transcription factors \cite{sanguinetti2006probabilistic}. Work in this direction would require extending the theory of SBPs to non constant volatility functions.

Another promising setting is when multiple frames of data are available rather than just two boundary conditions. The IPFP algorithm trivially adapts to multiple constraints \cite{cramer2000probability} rather than just initial and terminal distributions. Future work could explore experiments similar to the one presented in \cite{tong2020trajectorynet} where multiple frames are considered during training and the performance is measured using a leave-one-out procedure.  

% Finally, in this work, we did not make use of the uncertainty estimates provided by the GP. One interesting direction would be to leverage the uncertainty estimates to encourage exploration when simulating and fitting the half bridge problems in order to have further control over the splitting of the drift (similarly to the double well experiment). 

%Furthermore due to computational limitations another aspect of GPs which we did not explore was the fitting of hyperparameters, further work for scaling and adapting hyper-parameter fitting to the IPML iterations could result in considerable improvements to this approach.

\subsection{Conclusion} 
We have presented IPML, a method to solve the Schrödinger bridge for arbitrary diffusion priors. We presented theoretical results guaranteeing convergence in the limit of infinite data. We devised a practical application of the algorithm using Gaussian Processes and presented several experiments on a variety of problems from synthetic to biological data. The approach opens up opportunities in science, where oftentimes, prior knowledge about the temporal evolution of a process has been developed but needs to be combined with data-driven methods to scale up to modern problems.

\subsection{Acknowledgements} 

We would like to thank Abdul Fatir Ansari for insightful discussion and helpful remarks.

% \subsection{References} 

% \externalbibliography{yes}
\bibliographystyle{aaai}
\bibliography{bib_2021}

\appendix
\onecolumn

\section*{Appendix}
% \makeatletter
% \def\CTEX@section@format{\Large\bfseries}
% \makeatother

\section{Assumptions Across Proofs}

For abstraction purposes we will list the set of assumptions assumed across all results in this appendix:
\begin{itemize}
    \item All SDEs considered have $L-$Lipchitz diffusion coefficients as well as satisfying linear growth.
    \item The optimal drifts are elements of a compact space and thus satisfy the HJB equations. Note for the proposes of Theorem 1 this can be relaxed using notions of $\Gamma-$convergence.
    \item All SDEs satisfy a Fokker-Plank equation and thus $p(\vx,t)=\mathrm{Law}(X_t)$ is differentiable with respect to $\vx$. 
    \item The boundary distributions are bounded in $\mathscr{L}_{2}(\pi_i)$ that is  $\exists C< \infty, \; s.t. \; \E_{\pi_i}[|X|^2] \leq C$.
\end{itemize}

\section{Brief Introduction to the Schrödinger System and Potentials}\label{appdx:system}

The Schrödinger system and its potentials are mentioned when introducing some of the results and connections of the full Schrödinger bridge problem. In this section we will provide a brief introduction to how the system arises from the original static Schrödinger bridge. For brevity we will use $\vx = \vx(0)$ and $\vy=\vx(1)$ to denote the boundaries.

The static Schrödinger bridge problem can be derived from the full dynamic bridge by marginalising out the dynamics via the Disintegration Theorem and focusing on the problem only concerning what is happening at the boundaries. The static SBP is formulated as:

\begin{linenomath}
\begin{equation} \label{eq:static_bridge}
    \argmin_{q(\vx,\vy) \in \calD(\pi_0, \pi_1)} \KL(q || p^{\Q^{\gamma}})
\end{equation}
\end{linenomath}

Where $p^{\Q^{\gamma}}(\vx,\vy)$ is the prior joint distribution for the boundary and can be obtained by solving the FPK equation corresponding to the prior SDE. Now if we write down the Lagrangian for the above problem we arrive at:
\begin{linenomath}
\begin{align} \label{eq:static_bridge_lag}
   \KL(q || p^{\Q^{\gamma}}) + \int \lambda(\vx) \left[\int q(\vx,\vy)d\vy - \pi_0(\vx) \right]d\vx + \int \mu(\vy) \left[\int q(\vx,\vy)d\vx - \pi_1(\vy) \right]d\vy
\end{align}
\end{linenomath}

which via performing the appropriate variations wrt to $q$ leads to the optimal solution:
\begin{linenomath}
\begin{equation}
    q^{*}(\vx,\vy) = \exp\left(\ln p^{\Q^{\gamma}}(\vx)-1-\lambda(\vx)\right) p^{\Q^{\gamma}}(\vy | \vx) \exp\left( -\mu(\vy)\right),
\end{equation}
\end{linenomath}

which when relabeling the terms containing the Lagrange multipliers we obtain:
\begin{linenomath}
\begin{equation}
    q^{*}(\vx,\vy) = \hat{\phi}(\vx)p^{\Q^{\gamma}}(\vy | \vx) \phi(\vy),
\end{equation}
\end{linenomath}

such that :

%\begin{linenomath}
%\begin{equation}
\begin{align}
        \hat{\phi}(\vx) \int p^{\Q^{\gamma}}(\vy | \vx) \phi(\vy)d\vy = \pi_0(\vx) \\
    {\phi}(\vy) \int p^{\Q^{\gamma}}(\vy | \vx) \hat{\phi}(\vx)d\vx = \pi_1(\vy) ,
\end{align}
%    \begin{split}

%    \end{split}
%\end{equation}
%\end{linenomath}
furthermore if we relabel the potentials to indicate the times they correspond to we arrive at the Schrödinger system:
\begin{align}
    \phi(\vx,0) \hat{\phi}(\vx, 0) = \pi_0(\vx) , \quad \phi(\vy,1) \hat{\phi}(\vy, 1) = \pi_1(\vy),
\end{align}
where
\begin{align}
    \hat{\phi}(\vx, 0) = \hat{\phi}(\vx) , \quad & \phi(\vy,1)=  \phi(\vy), \\
    \phi(\vx,0) = \int p^{\Q^{\gamma}}(\vy | \vx) \phi(\vy)d\vy,\quad & \hat{\phi}(\vy,1) = \int p^{\Q^{\gamma}}(\vy | \vx) \hat{\phi}(\vx)d\vx .
\end{align}
The above functional system is refered to as the Schrödinger system and the Schrödinger potentials are given by $\phi, \hat{\phi}$. Furthermore the time interpolates for the potentials can be obtained by the propagation's: 
\begin{align}
            \phi(\vx, t) &= \!\!\!\int \!\!\phi(\vz(1)) p^{\Q_0^{\gamma}}(\vz(1) | \vx(t)) d\vz(1) \\
            \hat{\phi}(\vy, t) &= \!\!\!\int\!\! \hat{\phi}(\vz(0)) p^{\Q_0^{\gamma}}(\vy(t) | \vz(0)) d\vz(0) 
\end{align}
For a more rigorous and extensive introduction please see \cite{pavon1991free}.

\section{Disintegration Theorem - Product Rule for Measures}\label{app:desint}

In this section, we present the Disintegration Theorem in the context of probability measures, which serves as the extension of the product rule to measures that do not admit the traditional product rule. Furthermore we will provide a proof for a direct lemma of the Disintegration Theorem that is more analogous to the standard product rule. Like the product rule these theorems are essential for decomposing and manipulating path measures and thus is needed for most results pertaining to the dynamic SBP.

\begin{theorem} (Disintegration Theorem for continuous probability measures): 

For a probability space $(Z, \calB(Z) ,\P)$, where $Z$ is a product space: $Z = Z_x \times Z_y$ and
\begin{itemize}
    \item  $Z_x \subseteq \R^d, Z_y \subseteq \R^{d'}$,
    \item  $\pi_i: Z \rightarrow Z_i$ is a measurable function known as the canonical projection operator (i.e. $\pi_x(z_x,z_y) = z_x$ and $\pi^{-1}_x(z_x) = \{y | \pi_x(z_x) = z\}$),
\end{itemize}
there exists a measure $\P_{y|x}(\cdot | \vx)$, such that
  \begin{align}
      \int_{Z_x \times Z_y} f(\vx, \vy) d\P(\vy) = \int_{Z_x}\int_{Z_y} f(\vx,\vy) d\P_{y|x}(\vy | \vx) d\P(\pi^{-1}(\vx)),
  \end{align}
 where $P_x(\cdot) = \P(\pi^{-1}(\cdot))$ is a probability measure, typically referred to as a push-forward measure, and corresponds to the marginal distribution.
\end{theorem}

A direct consequence of the above instance of the disintegration theorem is, with $f(\vx,\vy) = \ind_{A_x \times A_y}(\vx,\vy)$,
\begin{align}
    \P(A_x \times A_y) = \int_{A_x}\P(A_y | \vx) d\P_x(\vx).
\end{align}
 We can see that, in the context of probability measures, the above is effectively analogous to the product rule.

\subsection{RN Derivative Disintegration}

We now have the required ingredients to show the following:

\begin{lemma}\label{lemma:rn_des}(RN-derivative product rule)
Given two probability measures defined on the same product space,  $(Z_x \times Z_y, \calB(Z_x \times Z_y) ,\P)$ and $(Z_x \times Z_y, \calB(Z_x \times Z_y) ,\Q)$, the Radon–Nikodym derivative $\frac{d\P}{d\Q} (\vx,\vy)$ can be decomposed as
\begin{align}
    \frac{d\P}{d\Q} (\vx,\vy) = \frac{d\P_{y|x}}{d\Q_{y|x}}(\vy)\frac{d\P_x}{d\Q_x}(\vx).
\end{align}
\end{lemma}
\begin{proof}
Note that this is an instructive sketch for a well known result \citep{pavon1991free,pavon2018data,leonard2013survey,leonard2014some}. We assume that $d\Q_{y|x} >> d\P_{y|x} $ which is not a trivial fact yet it follows from the disintegration theorem in this particular setting (See Theorem 1.6 b) in \cite{leonard2014some}) .
Starting from
\begin{align*}
    \P(A_x \times A_y) =  \int_{A_x}\P(A_y | \vx)  d\P_{x}(\vx),
\end{align*}
we apply the Radon-Nikodym theorem to $\P(A_y | \vx)$ and then to $P_x$:
\begin{align*}
    \P(A_x \times A_y) &=  \int_{A_x}\int_{A_y} \frac{d\P_{y|x}}{d\Q_{y|x}}(\vy) d\Q_{y|x}(\vy) d\P_{x}(\vx) \\
     &= \int_{A_x}\left(\int_{A_y} \frac{d\P_{y|x}}{d\Q_{y|x}}(\vy) d\Q_{y|x}(\vy)\right) \frac{d\P_{x}}{d\Q_{x}}(\vx)d\Q_{x}(\vx) \\
      &= \int_{A_x}\int_{A_y}\frac{d\P_{x}}{d\Q_{x}}(\vx)\frac{d\P_{y|x}}{d\Q_{y|x}}(\vy) d\Q_{y|x}(\vy) d\Q_{x}(\vx).
\end{align*}
Now, via the disintegration we have that
\begin{align*}
  \int_{A_x \times A_y} \frac{d\P_{x}}{d\Q_{x}}(\vx)\frac{d\P_{y|x}}{d\Q_{y|x}}(\vy) d\Q(\vx,\vy) = \int_{A_x}\int_{A_y}\frac{d\P_{x}}{d\Q_{x}}(\vx)\frac{d\P_{y|x}}{d\Q_{y|x}}(\vy) d\Q_{y|x}(\vy) d\Q_{x}(\vx).
\end{align*}
Thus, we conclude that
\begin{align*}
    \P(A_x \times A_y) = \int_{A_x \times A_y} \frac{d\P_{x}}{d\Q_{x}}(\vx)\frac{d\P_{y|x}}{d\Q_{y|x}}(\vy) d\Q(\vx,\vy), 
\end{align*}
which, via the Radon-Nikodym theorem, implies
\begin{align*}
    \frac{d\P}{d\Q} (\vx,\vy) = \frac{d\P_{y|x}}{d\Q_{y|x}}(\vy)\frac{d\P_x}{d\Q_x}(\vx).
\end{align*}

Note that swapping $\Q$ for the Lebesgue measure would result in the standard product rule for probability density functions.
\end{proof}

% \section{Proof Sketches For Control Formulation}\label{app:control}
% \begin{proof}
% Via the Disintegration Theorem and Theorem \ref{lemma:rn_des}, we can condition on the endpoint and re-write the RN derivative as
% \begin{align*}
%     \frac{d\Q}{d\W^{ \gamma}} = \frac{\pi_0^\Q}{\pi_0^{ \W^{\gamma}}} \frac{d\Q_{(0,1]}}{d\W_{(0,1]}^{\gamma}}\left(\cdot | \rvx(0) =\vx\right),
% \end{align*}
% where the disintegration $\Q_{(0,1]}\left(\cdot | \rvx(0) = \vx \right)$ is a solution to $d\rvx(t) = \rvb_t^+ dt + \sqrt{\gamma} \rvw^+(t)$. Then, by Theorem \ref{thrm:ito_ratio}, we can express the RN derivative in terms of  the drift $\rvb_t^+$:
% \begin{align*}
%     \frac{d\Q}{d\W^{\gamma}} = \frac{\pi_0^\Q}{\pi_0^{ \W^{\gamma}}}\exp\left(\int_0^1\frac{1}{2\gamma} \big|\big|\rvb^+(t)\big|\big|^2 dt\right).
% \end{align*}
% Substituting the above back into the KL divergence completes the result for Theorem \ref{eq:free_energy_1}.
% \end{proof}
% \section{Brief Introduction to IPFP}

\section{Proof Sketches For Half Bridges}

In this section we provide a proof sketch for the closed form solution of the half bridges as well as a proof for Observation \ref{obs:half}.

\label{app:half}
\begin{theorem}\label{thrm:half_bridge_forward}
    The forward half bridge admits the following solution: 
\begin{align}
    {\P}^{*}\left(A_0 \times A_{(0,1]}\right) =  \int_{A_0\times A_{(0,1]}} \frac{d \pi_0}{ d\pi_0^{\Q}} d\Q.
\end{align}
\end{theorem}
\begin{proof}
Via the disintegration theorem, we have the following decomposition of KL:
\begin{align}
    \KL(\P || \Q) = \KL(\pi_0^\P|| \pi_0^{\Q} )  + \E_{\pi_0^\P}\left[\KL(\P(\cdot| \vx) || \Q(\cdot | \vx))\right]. \nonumber
\end{align}
Thus, via matching terms accordingly, we can construct $\P^*$ by setting $\P(\cdot| \vx)=\Q(\cdot| \vx)$ and matching the constraints:
\begin{align}
    {\P}^{*}\left(A_0 \times A_{(0,1]}\right) = \int_{A_0}\!\!\Q(A_{(0,1]} | \vx) d\pi_0(\vx),
\end{align}
\begin{align}
    {\P}^{*} &= \int_{A_0}  \frac{d\pi_0}{d\pi_0^{\Q}}(\vx)   \Q(\cdot | \vx) d \pi_0^{\Q}(\vx) \nonumber \\
    &= \int_{A_0 \times A_{(0,1]} }  \frac{d\pi_0}{d\pi_0^\Q}(\vx)  d \Q.
\end{align}
\end{proof}

\begin{repobservation}{obs:half}
We can parametrise a measure $\Q$ with its drift as the solution to either of the following SDEs:
\begin{align*}
    d\rvx^{\pm}&(t) =  \rvb^{\pm}(t) + \sqrt{\gamma} d \rvw^{\pm}(t) ,\\
 &\rvx^+(0) \sim \pi_0^\Q,\:\rvx^-(0) \sim \pi_1^\Q.
\end{align*}
Then, we can sample from the solution to the following half bridges:
\begin{align*}
   \P^{*-} = \arginf_{\P \in \calD(\cdot,\pi_1)}\KL(\P || \Q), \\ 
   \P^{*+} = \arginf_{\P \in \calD(\pi_0,\cdot )}\KL(\P || \Q),
\end{align*}
via simulating trajectories following the SDEs
\begin{align*}
     d\rvx^{-}(t) &=   \rvb^{-}(t) + \sqrt{\gamma} d \rvw^{-}(t) , \quad  \rvx^{-}(0) \sim \pi_1, \\
     d\rvx^{+}(t) &=  \rvb^{+}(t) + \sqrt{\gamma} d \rvw^{+}(t) , \quad  \rvx^{+}(0) \sim \pi_0.
\end{align*}
Paths sampled from the above SDEs will be distributed according to $\P^{*-}$ and $\P^{*+}$ respectively.
\end{repobservation}

\begin{proof}(Sketch)
W.l.o.g., Consider the decomposition of the KL divergence that follows from the disintegration Theorem (Appendix \ref{app:desint}):
\begin{align}
    \KL(\P || \Q) = &\KL(\pi_0^\P|| \pi_0^{\Q} ) \nonumber  + \E_{\pi_0^\P}\left[\KL(\P(\cdot| \vx) || \Q(\cdot | \vx))\right]. \nonumber
\end{align}
Furthermore the disintegration $\Q( \cdot| \rvx(0) )$ is a solution to the dynamics $d\rvx^{+}(t) =  \rvb^{+}(t) + \sqrt{\gamma} d \rvw^{+}(t)$. We can make the term $\E_{\pi_0^\P}\left[\KL(\P(\cdot| \vx) || \Q(\cdot | \vx))\right]$ go to 0 by setting $\P( \cdot| \rvx(0) ) =\Q( \cdot| \rvx(0) )$, it is clear the dynamics of $\P( \cdot| \rvx(0) )$ follows  $d\rvx^{+}(t) =  \rvb^{+}(t) + \sqrt{\gamma} d \rvw^{+}(t)$. What is left is to attach the constraint via $\rvx(0) \sim \pi_0(\rvx(0))$ which brings $\KL(\pi_0^\P|| \pi_0^{\Q} )$ to  $\KL(\pi_0 || \pi_0^{\Q})= \KL(\P^{*+} || \Q)$ coinciding with the half bridge minima as per \cite{pavon1991free,bernton2019schr}.

This is simply enforcing the constraints via the product rule (Disintegration Theorem) and then matching the remainder of the unconstrained interval with the disintegration for the reference distribution $\Q$, following Equation \ref{eq:half_bridge_forward}.
\end{proof}

\section{Proof Sketch for Reverse-MLE Consistency}  \label{app:cons}

The proof sketch for Theorem \ref{theorem:mle-rev} will show how the likelihood converges in the large data and small time step limit to an optimisation of the KL divergence between two reverse time diffusions, from here one can use the standard arguments to show this quantity is minimised when the two measures describe the same stochastic process or equivalently when the drifts are equal.

\begin{lemma}\label{lemma:normalization}
Normalizing the time reversed likelihood with the true discretised backwards SDE density does not affect the maximum likelihood estimate. That is :
% \end{paracol}
\begin{align}
    \argmax_{\vb} &\sum_n \sum_{k=0}^T \ln \gN \left( \vx^{(n)+}_{t_k- \Delta t} \Big| \vmu^{(n)+}_{t_k- \Delta t} , \gamma \Delta t\right) \nonumber \\ = \argmax_{\vb} &\frac{1}{N}\sum_n \sum_{k=0}^T \left(\ln \gN \left( \vx^{(n)+}_{t_k- \Delta t} \Big| \vmu^{(n)+}_{t_k- \Delta t} , \gamma \Delta t\right) -\ln  \gN \left( \vx^{(n)+}_{t_k- \Delta t} \Big| \vx^{(n)+}_{t_k } -  \Delta t\vb^{-}_{0}\left(\vx^{(n)+}_{t_k},1- t_k\right)  , \gamma \Delta t\right) \right)
\end{align}
% \begin{paracol}{2}
% \linenumbers
% \switchcolumn
\end{lemma}
\begin{proof}
The term $-\ln  \gN \left( \vx^{(n)+}_{t_k- \Delta t} \Big| \vx^{(n)+}_{t_k } -  \Delta t\vb^{-}_{0}\left(\vx^{(n)+}_{t_l}, 1-t_k \right)  , \gamma \Delta t\right)$ does not depend on $\vb$ and thus is an additive constant.
\end{proof}

\begin{reptheorem}{theorem:mle-rev}
(Consistency of Reverse-MLE Formulations) Let $\left\{ \big(\vx^{(n)+}_{t_k} \big)_{k=0}^T \right\}_{n=0}^{N}$ be sampled/discretised trajectories from the SDE that represents the half bridge measure $\P^+$:
\begin{align*}
     d\rvx^{+}(t) &=  \vb^{+}_0(\rvx^{+}(t),t) + \sqrt{\gamma} d \rvw^{+}(t) , \quad  \rvx(0) \sim \pi_0.
\end{align*}
Then carrying out maximum likelihood estimation of the time reversed drift ${\vb_0}^{-}$ on time reversed samples:
\begin{align}\label{eq:mle2}
   \prod_n p&\left( \big(\vx^{(n)+}_{t_k}\big)_{k=0}^T  \Big| \vb^{-}\right) \propto\\ \nonumber
    &\prod_n p(\rvx_{t_T}^{(n)+})\prod_{k=0}^T \gN \left( \vx^{(n)+}_{t_k- \Delta t} \Big| \vmu^{(n)+}_{t_k- \Delta t} , \gamma \Delta t\right),
\end{align}
where,
\begin{align}
    &\vmu^{(n)+}_{t_k- \Delta t} = \vx^{(n)+}_{t_k } -  \Delta t\vb^{-}\left(\vx^{(n)+}_{t_k}, 1-t_k\right) 
\end{align}
and $\vb^{-}$ is the drift of our estimator SDE:
\begin{align}
    d\rvx^{-}(t) &=  \vb^{-}(\rvx^{-}(t), t) + \sqrt{\gamma} d \rvw^{-}(t) , \quad  \rvx^{-}(t)(0) \sim p,
\end{align}
converges in probability to the to the true dual drift $\vb^{-}_0(\vx,t)=\vb^{+}_0(\vx,t) -\gamma \nabla_{\vx} \ln p(\vx, t)$ where $\vb^{+}_0(x, t)$ is the optimal half bridge drift for $\P^{+}$ as $N \rightarrow \infty, \Delta t \rightarrow 0$. Under the assumptions:
\begin{itemize}
    \item The density $p(\vx,t)$ is differentiable with respect to $\vx$.
    \item The optimal drift lies in a compact space (This can be relaxed using notions of $\Gamma-$convergence).
    \item The prior drift coefficient is $L-$Lipchitz and satisfies linear growth.
\end{itemize}
\end{reptheorem}
\begin{proof}

For the interest of brevity let $\vb^{(n)-}_t = \vb^{-}\left(\vx^{(n)+}_{t_k}, 1-t_k \right)$. Taking logs and applying Lemma \ref{lemma:normalization} to Equation \ref{eq:mle2} yields:
% \newpage
% \end{paracol}
\begin{align}
    \frac{1}{N}\sum_n \sum_{k=0}^T \left(\ln \gN \left( \vx^{(n)+}_{t_k- \Delta t} \Big| \vmu^{(n)+}_{t_k- \Delta t} , \gamma \Delta t\right) -\ln  \gN \left( \vx^{(n)+}_{t_k- \Delta t} \Big| \vx^{(n)+}_{t_k } -  \Delta t\vb^{-}_{0}\left(\vx^{(n)+}_{t_k},1-t_k\right)  , \gamma \Delta t\right) + \ln\frac{p(\rvx_{t_T})}{q(\rvx_{t_T})}\right) \\
    \frac{1}{N} \sum_n \frac{1}{2\gamma\Delta t}\sum_{k=0}^T \left(-\Bigg|\Bigg| {\vx^{(n)+}_{t_k- \Delta t} -   (\vx^{(n)+}_{t_k } - \Delta t\vb^{(n)-}_t) }\Bigg|\Bigg|^2  +\Bigg|\Bigg|{\vx^{(n)+}_{t_k- \Delta t} - (\vx^{(n)+}_{t_k } - \Delta t\vb^{(n)-}_{0t}) }\Bigg|\Bigg|^2 + \ln\frac{p(\rvx_{t_T})}{q(\rvx_{t_T})}\right)
\end{align}
% \begin{paracol}{2}
% \linenumbers
% \switchcolumn

Where $q(\rvx_{t_T})$ represents the terminal distribution of the forward SDE (i.e. $\rvx^{+}(1) \sim \pi$).
Now we can equivalently write the above expression in terms of the time reversed samples (i.e. $ \vx_{t_i}^{-} = \vx_{t_{n-i}}^{+}$):
% \end{paracol}
\begin{align}
     \frac{1}{N}\sum_n \frac{1}{2\gamma\Delta t} \sum_{k=0}^T \left(-\Bigg|\Bigg| {\vx^{(n)-}_{t_k} -   (\vx^{(n)-}_{t_k- \Delta t} - \Delta t\vb^{(n)-}_t) }\Bigg|\Bigg|^2  +\Bigg|\Bigg| {\vx^{(n)-}_{t_k} - (\vx^{(n)}_{t_k - \Delta t} - \Delta t\vb^{(n)-}_{0t}) }\Bigg|\Bigg|^2+ \ln\frac{p(\rvx^{(n)-}
_{t_0})}{q(\rvx^{(n)-}
_{t_0})}\right)
\end{align}
% \begin{paracol}{2}
% \linenumbers
% \switchcolumn
Expanding the squares and re-arranging:
% \end{paracol}
\begin{align} \label{eq:maximand}
     \frac{1}{N}\sum_n \frac{1}{\gamma } \sum_{k=0}^T \left( { (\Delta \vx^{(n)-} )^\top (\vb^{(n)-}_{1-t} - \vb^{(n)-}_{0(1-t)} )} - \frac{1}{2} (||\vb^{(n)-}_{1-t}||^2 - ||\vb^{(n)-}_{0(1-t)}||^2 )\Delta t + \ln\frac{p(\rvx^{(n)-}
_{t_0})}{q(\rvx^{(n)-}
_{t_0})}\right),
\end{align}
% \begin{paracol}{2}
% \linenumbers
% \switchcolumn
where $\Delta \vx^{(n)-} = \vx^{(n)-}_{t_k} -   \vx^{(n)-}_{t_k -\Delta t}$. Now we can consider the limit of Equation \ref{eq:maximand}:
% \end{paracol}
\begin{align} 
     \lim_{N \rightarrow \infty}\lim_{\Delta t \rightarrow 0}\frac{1}{N}\sum_n \frac{1}{\gamma } \sum_{k=0}^T \left( { (\Delta \vx^{(n)-} )^\top (\vb^{(n)-}_{1-t} - \vb^{(n)-}_{0(1-t)} )} - \frac{1}{2} (||\vb^{(n)-}_{1-t}||^2 - ||\vb^{(n)-}_{0(1-t)}||^2 )\Delta t + \ln\frac{p(\rvx^{(n)-}
_{t_0})}{q(\rvx^{(n)-}
_{t_0})}\right),
\end{align}
% Under uniform convergence assumptions of the parametrised drift I think it is we can swap maxes and limits (TODO: check a more thorough MLE consistency proof for conditions, should we be assuming $\Gamma-$convergence ?).
\begin{align} 
    \lim_{N \rightarrow \infty}\frac{1}{N}\sum_n \frac{1}{\gamma }\lim_{\Delta t \rightarrow 0}\sum_{k=0}^T \left( { (\Delta \vx^{(n)-} )^\top (\vb^{(n)-}_{1-t} - \vb^{(n)-}_{0(1-t)} )} - \frac{1}{2} (||\vb^{(n)-}_{1-t}||^2 - ||\vb^{(n)-}_{0(1-t)}||^2 )\Delta t+   \ln\frac{p(\rvx^{(n)-}
_{t_0})}{q(\rvx^{(n)-}
_{t_0})} \right),
\end{align}
% \begin{paracol}{2}
% \linenumbers
% \switchcolumn
We can write the sum over the time grid as a stochastic integral if we express the inner terms using the continuous time approximation from Lemma \ref{lemma:rev_eu}:
% \end{paracol}
\begin{equation} 
    \lim_{N \rightarrow \infty}\frac{1}{N}\sum_n \frac{1}{\gamma }\lim_{\Delta t \rightarrow 0}\int_{0}^1  {  (\vb^{(n)-}_{1-t} - \vb^{(n)-}_{0(1-t)} )^{\top} d\hat{\vx}^{(n)-}(t) } - \frac{1}{2} \int_{0}^1 (||\vb^{(n)-}_{1-t}||^2 - ||\vb^{(n)-}_{0(1-t)}||^2 )dt +   \ln\frac{p(\rvx^{(n)-}
_{t_0})}{q(\rvx^{(n)-}
_{t_0})}  ,
\end{equation}
% \begin{paracol}{2}
% \linenumbers
% \switchcolumn
% Now using  that $\hat{\vx}^{(n)-}(t) \overset{P}{\rightarrow} {\vx}^{(n)-}(t)$ from Lemma \ref{lemma:rev_eu} as $\Delta  t \rightarrow 0$ it. For Euler part
From \citep{nelson1967dynamical,follmer1984entropy,anderson1982reverse} it follows that  $\hat{\vx}^{(n)-}(t)$ is a semi-martingale (w.r.t. to the backwards filtration see \cite{elliott1985reverse,kunitha1982backward}) then in the limit the stochastic integrals are taken with respect to the true time reversed stochastic process adapted to the backwards filtration $(\calF^{-}_i)_{i\in T}$ (see Theorem 2.13 in \cite{revuz2013continuous}): % ,for another formal argument of Equation \ref{eq:stoch_rev} see Theorem 3.17 and the equation following Equation 2.3 in \cite{cattiaux2021time}):
% \end{paracol}
\begin{align}
    \lim_{\Delta t \rightarrow 0}\int_{0}^1  {  (\vb^{(n)-}_{1-t} - \vb^{(n)-}_{0(1-t)} )^{\top} d\hat{\vx}^{(n)-}(t) }  &\xrightarrow{P}\int_{0}^1  {  (\vb^{-}(\vx^{(n)-}(t), t) - \vb^{-}_{0}(\vx^{(n)-}(t), t) )^{\top} d{\vx}^{(n)-}(t) } \label{eq:stoch_rev}  \\
    \lim_{\Delta t \rightarrow 0}\int_{0}^1 (||\vb^{(n)-}_{1-t}||^2 - ||\vb^{(n)-}_{0(1-t)}||^2 )dt &\xrightarrow{P} \int_{0}^1 (||\vb^{-}(\vx^{(n)-}(t), t) ||^2 - ||\vb^{-}_{0}(\vx^{(n)-}(t), t) ||^2 )dt ,
\end{align}
% \begin{paracol}{2}
% \linenumbers
% \switchcolumn
now we have:
% \end{paracol}
\begin{align} 
    \lim_{N \rightarrow \infty}\frac{1}{N}\sum_n& \frac{1}{\gamma }\int_{0}^1  {  (\vb^{-}(\vx^{(n)-}(t), t) - \vb^{-}_{0}(\vx^{(n)-}(t), t) )^{\top} d{\vx}^{(n)-}(t) }  \nonumber \\ &- \frac{1}{2\gamma} \int_{0}^1 (||\vb^{-}(\vx^{(n)-}(t), t) ||^2 - ||\vb^{-}_{0}(\vx^{(n)-}(t), t) ||^2 )dt + \ln\frac{p(\rvx^{(n)-}
(0))}{q(\rvx^{(n)-}
(0))} ,
\end{align}
% \begin{paracol}{2}
% \linenumbers
% \switchcolumn
where by Lemma \ref{lemma:rev_eu} each random function $\rvx^{(n)-}(t)$ is sampled i.i.d from the SDE:
\begin{align} \label{eq:convrged_sde_back}
     d\rvx^{-}(t) &=  \vb^{-}_0(\rvx^{-}(t),t) + \sqrt{\gamma} d \rvw^{+}(t) , \quad  \rvx^{-}(0) \sim q,
\end{align}
we can now apply the weak law of large numbers (WLLN):
\begin{align} 
    \E_{\Q}& \Bigg[\frac{1}{\gamma }\int_{0}^1  {  (\vb^{-}(\vx^{(n)-}(t), t) - \vb^{-}_{0}(\vx^{(n)-}(t), t) )^{\top} d{\vx}^{(n)-}(t) }  \nonumber \\
    &- \frac{1}{2\gamma} \int_{0}^1 (||\vb^{-}(\vx^{(n)-}(t), t) ||^2 - ||\vb^{-}_{0}(\vx^{(n)-}(t), t) ||^2 )dt + \ln\frac{p(\rvx^{(n)-}
(0))}{q(\rvx^{(n)-}
(0))}\Bigg],
\end{align}
Now using that the log RN derivative between the estimator SDE and the true SDE is given by Girsanov's theorem \citep{kailath1971structure,pavon1991free,sottinen2008application}:
\begin{align}
    \ln \frac{d \Q}{d\P_{\vb}} = &\frac{1}{\gamma }\int_{0}^1  {  ( \vb^{-}_{0}(\vx^{(n)-}(t), t)-\vb^{-}(\vx^{(n)-}(t), t)  )^{\top} d{\vx}^{(n)-}(t) } \nonumber\\
    &- \frac{1}{2\gamma} \int_{0}^1 (||\vb^{-}_{0}(\vx^{(n)-}(t), t) ||^2-||\vb^{-}(\vx^{(n)-}(t), t) ||^2  )dt - \ln\frac{p(\rvx^{(n)-}
(0))}{q(\rvx^{(n)-}
(0))},
\end{align}
we arrive at :
\begin{align} 
     \E_{\Q}& \Bigg[-\ln  \frac{d \Q}{d\P_{\vb}}\Bigg] =  -D(\Q || \P_{\vb}).
\end{align}
Thus using $\frac{1}{N}\mathcal{L}(\vb)$ to denote the negative normalized log-likelihood we have shown the following pointwise (i.e. for $\vb \in \mathcal{B}$) convergence in probability: 
\begin{align}
   \left| \frac{1}{N}\mathcal{L}(\vb) - D(\Q || \P_{\vb}) \right| \xrightarrow{P} 0 .
\end{align}
Then If $\mathcal{B}$ , is compact plus additional continuity and boundedness assumptions on $\mathcal{L}(\vb)$ a stronger form of uniform convergence in probability can be attained:
\begin{align}
   \sup_{b \in \mathcal{B}}\left| \frac{1}{N}\mathcal{L}(\vb) - D(\Q || \P_{\vb}) \right| \xrightarrow{P} 0 ,
\end{align}
from the above it directly follows that:
\begin{align}
  \left| \frac{1}{N}\mathcal{L}(\vb^{*}_N) - D(\Q || \P_{\vb^{*}_N,}) \right| \xrightarrow{P} 0 ,
\end{align}
where $\vb^{*}_N$ is the MLE estimate of the dual drift at $N$ samples. Which implies:
\begin{align}
\lim_{N\xrightarrow{P}\infty }D(\Q ||& \P_{\vb^{*}_N})  = 0, \\
\lim_{N\xrightarrow{P}\infty } \vb^{*}_N &= \vb_0^{-}
\end{align}
See Chapter 4.5 of \cite{levy2008principles} for a more detailed discussion on the required assumptions of $\mathcal{L}$.
\end{proof}
% \end{proof}

\subsection{Extending Theorem \ref{theorem:mle-rev} to EM Samples}

Note that the above proof holds for discretised samples from the original SDE, however we have been unable to extend it to approximate sampling schemes such as EM. We believe that the following result motivates the possibility that Theorem \ref{theorem:mle-rev} holds when the samples are obtained from a consistent scheme such as EM:

\begin{lemma}\label{lemma:rev_eu} (Convergence of discrete time reversal)
The time reversal of discrete Euler-Mayurama samples :
\begin{align}
    \vx_{t_i}^{-} = \vx_{t_{n-i}}^{+}
\end{align}
converges in probability to the solutions of the time reversed diffusion:
\begin{align}
    d\rvx^{-}(t)& = \rvb^{-}(t) dt + \sqrt{\gamma} d\rvw^{-}(t)
\end{align}
where $\rvx^{-}(t)=\rvx^{+}(1-t)$
\end{lemma}

\begin{proof}
First let's consider the continuous time step-wise approximation induced by the the EM samples:
\begin{align*}
    \hat{\rvx}(t) = \sum_{i=1}^T \vx_{t_{i-1}} \delta_{t \in [t_{i-1}, t_i)},
\end{align*}

From \cite{gyongy1996existence} it is a well known result that under the standard regularity assumptions on the drift $\vb(\vx,t)$ (i.e. Lipchitz continuity in $t$ and and $\vx$) that the above approximation converges in probability to the SDE solution $\rvx^{+}(t)$ as $\Delta t \rightarrow 0$.

Observing that that the time reversal of the above corresponds to the approximation induced by the reverse samples we now consider the continuous time reversed approximation:
\begin{align*}
    \hat{\rvx}^{-}(t) = \hat{\rvx}(1-t)= \sum_{i=1}^T \vx_{t_{i-1} } \delta_{1-t \in [t_{i-1}, t_i)},
\end{align*}
Using $\hat{\rvx}(t)$ converges in probability to $\rvx^{+}(t)$ we have that $\forall t \in [0,1], \; \epsilon > 0$:
\begin{align}
    \lim_{\Delta t \rightarrow 0} P\left(\sup_{t \in[0,1]}\big|\hat{\rvx}(t) - \rvx^{+}(t)\big| > \epsilon\right) = 0,
\end{align}
we carry out the following substitution $s=1-t$ then $\forall s \in [0,1], \; \epsilon > 0$:
\begin{align}
    \lim_{\Delta t \rightarrow 0} &P\left(\sup_{t \in[0,1]}\big|\hat{\rvx}(1-s) - \rvx^{+}(1-s)\big| > \epsilon\right) = 0, \\
    \lim_{\Delta t \rightarrow 0} &P\left(\sup_{t \in[0,1]}\big|\hat{\rvx}^{-}(s) - \rvx^{-}(s)\big| > \epsilon\right) = 0,
\end{align}
which completes the proof. Note that showing that reversing backwards samples reversed to a forward direction converges in probability follows the same sketch structure.
\end{proof}

A potential proof strategy would be to try and exploit the strong convergence properties of the Euler scheme to show that the stochastic integral from Theorem \ref{theorem:mle-rev} also converges in the case of Euler samples.  This remains an interesting question for future work.

% It may be the case that a correction term needs to be introduced in order for the integrals built from the reversed EM sample to converge to the backwards integrals. . 

\section{Towards a Finite Sample Analysis of Approximate IPFP Schemes} \label{appdx:finite}

We use the term approximate IPFP schemes for methodologies such as the one we present in this paper (i.e. IPML) where steps 5,6 of IPFP (Algorithm \ref{alg:gipfp}) are replaced with inexact approximations. In this section we will present a rough sketch that takes the first step towards formally analysing approximate IPFP schemes in the finite sample / discretisation regime. This will serve to illustrate many of the challenges that still remain in this analysis as well as provide some initial results.

\begin{conjecture}(Heuristic Finite Sample Bound)

Given that:
    \begin{itemize}
        \item The exact IPFP at the $i^{th}$ iteration can be bounded from above as:
            \begin{align}
                ||\vu^{+}_{*}- {\vu}^{+}_{i}|| \leq \frac{S^{\Q_0^\gamma}_{\pi_{0,1}}}{i}
            \end{align}
        \item The approximate IPFP projection operators $\hat{\gP}^{\pm}_{\pi} : \gX \rightarrow \gX $ are $K-$Lipchitz 
        \item and the approximate finite sample projection error can be bounded by a constant:
        \begin{align}
            ||\vu^{+}_{i}-{\hPr} [\vu^{+}_{i-1}]|| \leq\epsilon_{\Delta T, N},
        \end{align}
    \end{itemize}

Then it follows that the approximate IPFP iteration error can be bounded from above by:
\begin{align} \label{eq:one_bound}
    ||\vu^{+}_{*}-\hat{\vu}^{+}_{i}||  \leq \epsilon_{\Delta T, N}  i + \frac{S^{\Q_0^\gamma}_{\pi_{0,1}}}{i}
\end{align}
when $\gP^{\pm}_{\pi}$ are non-expansive operators ($K=1$), and:
\begin{align}
    ||\vu^{+}_{*}-\hat{\vu}^{+}_{i}||  \leq \epsilon_{\Delta T, N}  \frac{1-K^{i}}{1-K} + \frac{S^{\Q_0^\gamma}_{\pi_{0,1}}}{i}
\end{align}
when $K\neq1$. Where $S^{\Q_0^\gamma}_{\pi_{0,1}} = \inf_{\Q \in \calD(\pi_0, \pi_1)} \KL\left(\Q \big|\big| \Q_0^{\gamma}\right),$ and $|| . ||$ is the $\mathscr{L}_{2}(\Q^{*})$ norm error as per Equation \ref{eq:free_energy_11}, that is it is just the half bridge in Line 6 of Algorithm \ref{alg:gipfp} written in terms of the drifts. 
\end{conjecture}
\begin{proof}(Sketch)

    For notational simplicity we will relabel the half bridges in Algorithm \ref{alg:gipfp} as projection operators $\gP^{\pm}_{\pi} : \gX \rightarrow \gX $  in a function space $\calX$. This leads to the following iterates:
    \begin{align}
        \vu^{-}_{i} &= \gP^{-}_{\pi_1}[\vu^{+}_{i-1}] \\
        \vu^{+}_{i} &= \gP^{+}_{\pi_0}[\vu^{-}_{i}]
    \end{align}
    And for the approximate IPFP we have the iterates:
    \begin{align}
        \hat{\vu}^{-}_{i} &= \hat{\gP}^{-}_{\pi_1}[\hat{\vu}^{+}_{i-1}] \\
        \hat{\vu}^{+}_{i} &= \hat{\gP}^{+}_{\pi_0}[\hat{\vu}^{-}_{i}].
    \end{align}
    Furthermore to simplify the analysis further we will consider the composition of the two projection operators (i.e. $\Pr = \gP^{+}_{\pi_0} \circ \gP^{-}_{\pi_1}$ )combined into a single operator (for both exact and approximate projections):
     \begin{align}
        \vu^{+}_{i} =\Pr[\vu^{+}_{i-1}] = \gP^{+}_{\pi_0}[\gP^{-}_{\pi_1}[\vu^{+}_{i-1}]]
    \end{align}
    We will now proceed to bound the error:
    \begin{align}
        ||\vu^{+}_{*}-\hat{\vu}^{+}_{i}|| \leq ||\vu^{+}_{i}-\hat{\vu}^{+}_{i}|| + ||\vu^{+}_{*}-{\vu}^{+}_{i}||
    \end{align}
    From the assumptions it follows that:
    \begin{align}
        ||\vu^{+}_{*}- {\vu}^{+}_{i}|| \leq \frac{S^{\Q_0^\gamma}_{\pi_{0,1}}}{i}
    \end{align}
    yielding:
    \begin{align}
        ||\vu^{+}_{*}-\hat{\vu}^{+}_{i}|| \leq ||\vu^{+}_{i}-\hat{\vu}^{+}_{i}|| + \frac{S^{\Q_0^\gamma}_{\pi_{0,1}}}{i},
    \end{align}
    Now we will proceed to analyse the first term :
    \begin{align}
        ||\vu^{+}_{i}-\hat{\vu}^{+}_{i}|| &\leq ||\vu^{+}_{i}-{\hPr} [\vu^{+}_{i-1}]||  +  || {\hPr} [\vu^{+}_{i-1}]- \hat{\vu}^{+}_{i} ||\\
        &\leq ||\vu^{+}_{i}-{\hPr} [\vu^{+}_{i-1}]||  +  || {\hPr} [\vu^{+}_{i-1}] -\hPr[\hat{\vu}^{+}_{i-1}] || \\
         &\leq\epsilon_{\Delta T, N} +  K||\vu^{+}_{i-1}-\hat{\vu}^{+}_{i-1} ||
    \end{align}
    Now we can expand the recurrence until the first iteration, yielding:
     \begin{align}
        ||\vu^{+}_{i}-\hat{\vu}^{+}_{i}|| & \leq \sum_{j=0}^{i} \epsilon_{\Delta T, N} K^{j}     \end{align}
    When $K \neq 1$ we attain the bound:

    \begin{align}
        ||\vu^{+}_{i}-\hat{\vu}^{+}_{i}||  \leq \epsilon_{\Delta T, N}\frac{(1-K^{i})} {1-K} 
    \end{align}
    When the projection operator $\hPr$ is non expansive this gives the bound:
    \begin{align}
        ||\vu^{+}_{i}-\hat{\vu}^{+}_{i}|| & \leq \sum_{j=0}^{i} \epsilon_{\Delta T, N} K^{j} \\
        &  \leq \epsilon_{\Delta T, N}  i 
    \end{align}
\end{proof}

Note that the above sketch is more of a strategy to outline the challenges and steps required for this analysis rather than a proof itself. We believe that the formal finite sample analysis on approximate IPFP schemes merits its own separate work and is thus not the focus of this work however in order to highlight its importance we have given a template towards this formal analysis following a strategy similar to \cite{sra2012scalable}. This analysis allows us to understand what conditions are needed to be shown and how they will affect the finite sample convergence rates:

We will now proceed to discuss each of the assumptions that are required as elements of the above analysis strategy. 

\begin{assumption}
The exact IPFP at the $i^{th}$ iteration can be bounded from above as:
\begin{align}
    ||\vu^{+}_{*}- {\vu}^{+}_{i}|| \leq \frac{S^{\Q_0^\gamma}_{\pi_{0,1}}}{i}
\end{align}
\end{assumption}

This assumption/conjecture is somewhat reasonable since Proposition 2.1 in \cite{bernton2019schr} proves formally that the sum of the KL errors of the $i^{th}$ marginals are bounded by $\frac{S^{\Q_0^\gamma}_{\pi_{0,1}}}{i}$, furthermore via the disintegration theorem one can show that the half bridge error terms are in fact bounded by the sum of the error of the $i^{th}$ marginals \citep{pavon1991free}. Thus it seems feasible to extend the result in \cite{bernton2019schr} to the path error we use in our conjecture. 

\begin{assumption}
The approximate IPFP projection operators $\hat{\gP}^{\pm}_{\pi} : \gX \rightarrow \gX $ are $K-$Lipchitz . Specifically we are interested in the cases where they may be non-expansive and contractive.
\end{assumption}
This assumption is specific to the approximation method used to estimate the half-bridges. It requires going into rigorous details on the nature of the approximation and imposing suitable regularity assumptions. Not that ideally we want the operator to be contractive or at the very least no expansive since for Lipchitz constants greater than one the cumulative error would potentially accumulate exponentially.

It is interesting to note that in the contractive case we can attain the bound:
\begin{align}
    ||\vu^{+}_{*}-\hat{\vu}^{+}_{i}||  \leq  \frac{\epsilon_{\Delta T, N} }{1-K} + \frac{S^{\Q_0^\gamma}_{\pi_{0,1}}}{i}
\end{align}
where the error term is constant with respect to the iteration number. Additionally the bound in Equation \ref{eq:one_bound} also trivially applies.

\begin{assumption}
The approximate finite sample projection error can be bounded by a constant:
\begin{align}
    ||\vu^{+}_{i}-{\hPr} [\vu^{+}_{i-1}]|| \leq\epsilon_{\Delta T, N},
\end{align}
\end{assumption}

This should typically be something that one shows however as seen in \cite{sra2012scalable} it is something that is often assumed in these proof strategies. The reason for assuming the result is that the projection varies depending on how we approximate the KL minimisation's, in our case we combine our proposed IPML with drift estimation via Gaussian processes. The error of the drift estimation (i.e. $\epsilon_{\Delta T, N}$) via this approach is not straightforward to obtain and is still an area of active research see \citep{pokern2013posterior,papaspiliopoulos2012nonparametric} for more details. As mentioned earlier this is not the focus of our work thus it is reasonable for us to assume that the drift estimation machinery proposed by prior authors is sound and should be possible to bound with a reasonable error.

\subsection{GPDriftFit Implementation Details} \label{appdx:predmean}
In this section we provide the closed form predictive formulas used by our method to estimate the drift. Our GP formulation for the half bridge approximations yield the following mean per dimension $d$ estimates of the drift:
\begin{align}
[\bar{\vb}^{+}(\vx, t)]_d \!= \!\vec\big(\vk_d^{+}(\vx \!\oplus\! t)\big)^{\top} \!\!\left(\tilde{\mK}_d^{+} \!+\! \frac{\gamma}{\Delta t} \mathbb{I}_{MT}\right)^{-1}\!\!\!\!\!\vec(\mY_d^{+}), \\
[\vb^{-}(\vx, t)]_d \!= \!\vec\big(\vk_d^{-}(\vx\! \oplus \!t)\big)^{\top} \!\!\left(\tilde{\mK}_d^{-}\! +\! \frac{\gamma}{\Delta t} \mathbb{I}_{NT}\right)^{-1}\!\!\!\!\vec(\mY_d^{-}) ,
\end{align}
where
\begin{align*}
   & [ \tilde{\mK}_d^{\pm} ]_{i\cdot T+l,i'\cdot T+l'} = [ {\mK}_d^{\pm} ]_{ii'll'},\\
   &[ \mK_d^{+} ]_{mm'll'} \!= \!{K}_d\big(\vx^{(m)-}_{t_l\!-\!\Delta t}\!\oplus\! (t_l\!-\!\Delta t), \vx^{(m')-}_{t_{l'} - \Delta t}\!\oplus\!( t_{l'}\!- \!\Delta t) \big), \\
   &[ \mK_d^{-} ]_{nn'kk'} = {K}_d\big(\vx^{(n)+}_{t_k} \oplus t_k , \vx^{(n')+}_{t_{k'}}\oplus t_{k'} \big),  \\
   &[\vk_d^{+}(\vx \oplus t)]_{lm} = K_d(\vx \oplus t , \vx^{(m)-}_{t_l- \Delta t}\oplus (t_l - \Delta t)), \\
   &[\vk_d^{-}(\vx \oplus t)]_{kn} = K_d(\vx \oplus t , \vx^{(n)-}_{t_k}\oplus (t_k)), \\
   &[\mY_d^{\pm}]_{li} =  \left[\frac{\pm\vx^{(i)\mp}_{t_l} \mp \vx^{(i)\mp}_{t_l - \Delta t}}{\Delta t}\right]_d, %\quad [\mY_d^{-}]_{kn} =  \left[\frac{\vx^{(n)+}_{t_k - \Delta t}  - \vx^{(n)+}_{t_k}}{\Delta t} \right]_d,
\end{align*}
where $\oplus$ is the concatenation operator, $\vec$ is the standard vectorisation operator, and $K_d: \R^{d+1} \times \R^{d+1} \rightarrow \R$ is a valid kernel function.

Similar to \cite{papaspiliopoulos2012nonparametric,ruttor2013approximate}, we are not making use of the predictive variances\footnote{Using the predictive mean as an estimate for the drift can also be interpreted as a form of kernel ridge regression under the empirical risk minimisation framework}. Instead, we simply use the predictive mean as an estimate of the drift and subsequently use that estimate to perform the EM method, thus effectively we could interpret this approach as a form of kernel ridge regression under the empirical risk minimisation framework.

\subsubsection{Coupled vs Decoupled Drift Estimators}

A careful reader may enquire if the decoupled drift parametrisation we have used imposes any limitations on the generality of the drifts that our method can estimate in contrast to an approach that couples/correlates drift dimensions. While it may seem counter-intuitive, correlating the GP outputs (i.e. coupling the drift dimensions) is more restrictive than the decoupled approach we have taken. This point is discussed and motivated in \citep{alvarezkernel,evgeniou2005learning}, where it is illustrated that the coupling of the drift outputs correspond to regularising the RKHS hypothesis space. So in fact the decoupled approach we take in this work is more flexible (less limited) than coupling the drift outputs as discussed in \citep{alvarezkernel,evgeniou2005learning}. It is however important to note that the regularisation effects of coupling may be desirable in certain physical systems where we wish to impose a constraint, however this would come at a large computational cost. Finally note that our method parametrises the drift $[\vb]_i = f_i(\vx, t)$ with a function that depends on all dimensions of the dynamical process $\rvx(t)$, that is each drift coordinate depends on the entire input space and thus no simplifying assumptions have been made in this setting.

 \section{Estimates Required by the Sinkhorn-Knop Algorithm}\label{app:sink}
 
 In order to adapt the Sinkhorn-Knop algorithm to a general prior we require the following considerations:
 
 \begin{itemize}
     \item The Sinkhorn-Knop algorithm requires computing the cost $C^{\Q_0^\gamma}_{ij} = \ln p^{\Q_0^{\gamma}}(\vy_i | \vx_j)$. For more general SDE priors whose transition densities are not available in closed form can be challenging:
    \begin{itemize}
        \item First we have to estimate the empirical distribution for $\ln p^{\Q_0^{\gamma}}(\vy_i | \vx_j)$ by generating multiple $\rvx(1)$ samples from the priors SDE for each $\vx_j$ in the dataset:
        \begin{align}
        d\rvx (t)  = \vb(\rvx(t),t)dt + \gamma d\rvw(t) , \quad \rvx(0)=\vx_j\nonumber
        \end{align}
        \item Once we have samples $\{\rvx_{k}\}_j$ we must carry out  density estimation for each $j$ in order to evaluate $\ln p^{\Q_0^{\gamma}}(\vy_i | \vx_j)$.
    \end{itemize}

\item The Sinkhorn-Knop algorithm produces discrete potentials which can be interpolated using the logsumexp formula. However how do we go from these static potentials to time dependant trajectories ?
    \begin{itemize}
        \item From \cite{pavon2018data} we can obtain expressions for the optimal drift as a function of the time extended potentials:
        \begin{align}\label{eq:potentialt}
            \phi(\vx, t) &= \!\!\!\int \!\!\phi(\vz(1)) p^{\Q_0^{\gamma}}(\vz(1) | \vx(t)) d\vz(1) \\
            \hat{\phi}(\vy, t) &= \!\!\!\int\!\! \hat{\phi}(\vz(0)) p^{\Q_0^{\gamma}}(\vy(t) | \vz(0)) d\vz(0) \label{eq:potentialt2}
        \end{align}
        The first integral can be approximated using the Montecarlo approximation via sampling from the prior SDE to draw samples from $p^{\Q_0^{\gamma}}(\vy(t) | \vz(0)) $. However computing the $\hat{\phi}$ potential integral is less clear and requires careful thought. 
        % One way could be applying Bayes theorem and sampling fro the reverse time SDE of the prior, however some density estimation would still be required to estimate $p^{\Q_0^{\gamma}}(\vz(0))$ resulting from the application of Bayes Theorem.
        \item Note that in order to estimate the drift we would have to simulate the SDE prior every time we want to evaluate the drift which itself will be run in another SDE simulation to generate optimal trajectories. 
    \end{itemize}
\end{itemize}

 \begin{figure}[!t]
    \centering
    \includegraphics[scale=0.4]{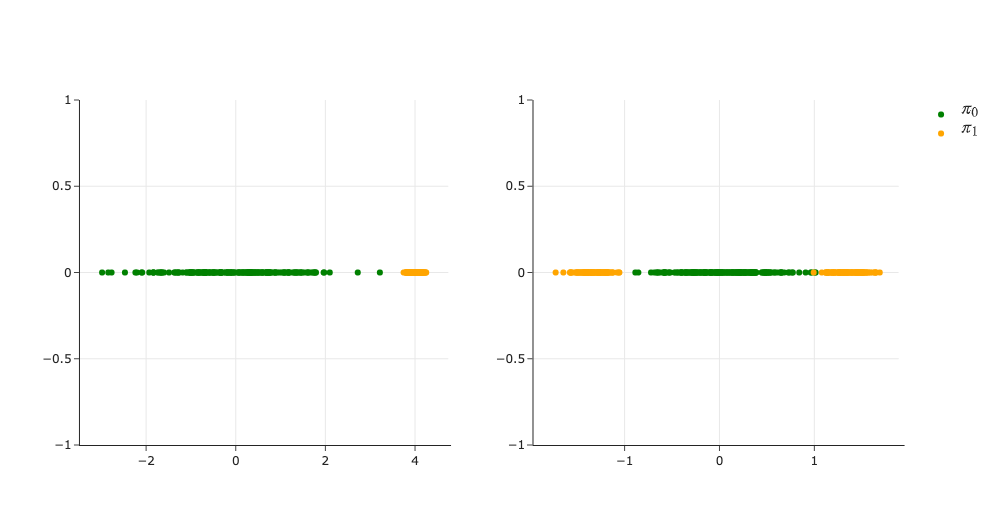}
    \caption{Unimodal experiment on the left and bimodal on the right.}
    \label{fig:toy_experiment}
\end{figure}

 \section{Approximations Required by (Pavon, Tabak, and Trigila 2018)} \label{app:pavonap}
 
 %\citet{pavon2018data}
 
 The approach in \cite{pavon2018data} iterates the following two objectives.
 \begin{align}
    \hat{\beta}^*_i = \argmax_{\hat{\beta} } &\frac{1}{M} \sum_s \log \hphi_0(\vx_s; \hat{\beta} )   \phi_0^{(i)}(\vx_s)\\
    &- \int \hphi_0(\vx; \hat{\beta} )   \phi_0^{(i)}(\vx) d\vx ,
    \quad \vx_s\sim \pi_0(\vx) ,  \nonumber 
\end{align}
and
\begin{align}
    {\beta}^*_i = \argmax_{{\beta} } &\frac{1}{N} \sum_s \log \phi_1(\vy_s; {\beta} )   \hphi_1^{(i)}(\vy_s)\\
    &- \int \phi_1(\vy; {\beta} )   \hphi_1^{(i)}(\vy) d\vy ,
    \quad \vy_s\sim \pi_1(\vy). \nonumber 
\end{align}
Where $\hat{\phi}_0, \phi_1$ are parametric functions aimed at estimating the potentials, and:
\begin{align} 
\int \hphi_0&(\vx; \hat{\beta} )   \phi_0^{(i)}(\vx) d\vx = \int \phi_1(\vy; {\beta} )   \hphi_1^{(i)}(\vy) d\vy\nonumber\\
& =\int  \hphi_0(\vx; \hat{\beta} ) \int p^{\Q_0^{\gamma}}(\vy| \vx) \phi_1^{(i)}(\vx; \beta) d\vy d\vx, \label{eq:hdint1}
\end{align}

Note that in Equation \ref{eq:hdint1} the outer most integral is not taken with respect to a probability distribution and thus does not admit standard approximations.  The authors in \cite{pavon2018data} propose the method of importance sampling \citep{martino2017effective}, however this does not scale well to higher dimensions. This is contrasting to our approach where all integrals are expectations with respect to the empirical distribution and the SDEs being fitted.
\begin{figure}[t]
    \centering
    \includegraphics[scale=0.4,trim={2.3cm 0.2cm 1.5cm 0}, clip]{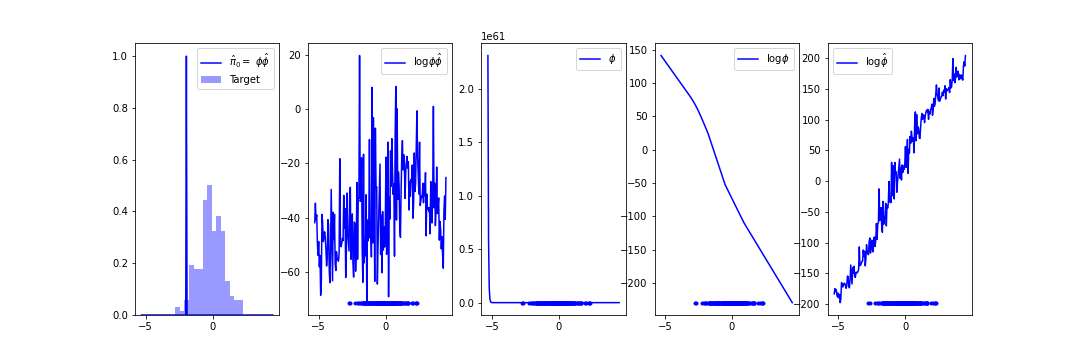} \\\vspace{-0.2cm}
    \includegraphics[scale=0.4,trim={2.3cm 0 1.5cm 1.5cm}, clip]{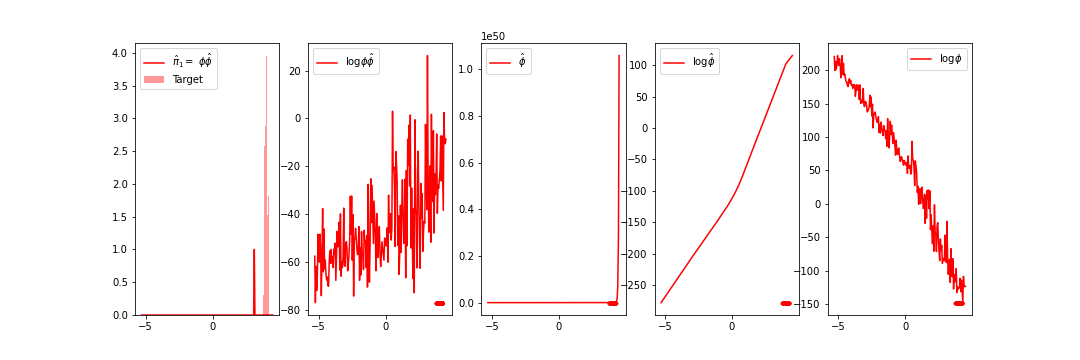} 
    \caption{Schrödinger Bridge results using the DDSB on unimodal to bimodal Gaussian 1D data. This example illustrates the Dirac delta collapse of the marginals.}
    \label{fig:small_delta_collapse}
\end{figure}

\section{Experiments} \label{app:experiments}

In this section we provide in depth details regarding the parameter configurations of our approach plus our experimental setup. For all the experiments below, unless stated otherwise, we used the exponential kernel with a lengthscale set at the default value of 1. We ran IPML for 5 iterations with a discretization factor $\Delta_t = 0.01$ and $\gamma=1$. The iteration number (5) was selected by observing that the algorithm converged after that number in all experiments considered.
%  IPML is implemented using a GP with an exponential kernel, $\gamma=1$, and over 5 iterations.
\subsection{1D Experiment}\label{sec:1d}
The prior $\Q_0^\gamma$ in IPML\footnote{The practical implementation of $\Q_0^\gamma$ will be added in the main paper during the revision period.} is specified by a prior drift and an initial value distribution $\pi_0^{\Q_0^\gamma} $.  We set $\pi_0^{\Q_0^\gamma}  = \pi_0$ and the prior to a Brownian motion $\W^\gamma$ unless stated otherwise. The initial value distribution does not affect the argmax of the SBP and is thus a free parameter. Note that $\pi_0^{\Q_0^\gamma}$ is used in practice to estimate the backwards drift of the prior before entering the IPFP loop by first sampling from $\Q_0^\gamma$ using the SDESolve subroutine. It is possible to encourage more exploration of the space by increasing the variance of $\pi_0^{\Q_0^\gamma}$.

For the unimodal boundary distributions we used:
\begin{align}
    \pi_0 \sim \calN\left(0,1\right), \;\; \pi_1 \sim \calN\left(4, 0.1^2\right)
\end{align}
and for the bimodal experiment we used:
\begin{align}
     \pi_0 \sim  \calN(0,  1) \quad \pi_1\sim \frac{1}{2}\calN(1.8, 0.6 ^2)+ \frac{1}{2}\calN(-1.9, 0.6 ^2) .
\end{align}

\subsubsection{Delta Collapse in DDSB} \label{appsec:delta}
Here we will illustrate a common failure case of the approach in \cite{pavon2018data}. That is when the distributions $\pi_0$ and $\pi_1$ are distant from each other the methodology proposed in \cite{pavon2018data} breaks for suitable values of $\gamma$ (i.e. $\gamma=1,2,3...$). Using the same marginals $\pi_0, \pi_1$ as in Experiment 1 we re-train the method by \cite{pavon2018data} with $\gamma=1$. In Figure \ref{fig:small_delta_collapse} we can see the results of this experiment. We can see the marginals learned by DDSB collapse all the mass in a very small low density region for the true marginals. We found this phenomena to occur often and harder to overcome in 2 dimensional experiments and thus did not pursue further comparisons with this method. Note we normalised the delta spike of the DDSB marginals to have a range in $[0,1]$ so we could compare to the true density, its actual value is on the order of $10^{11}$ which confirms that the model is collapsing to a point mass (dirac delta function).

\subsection{Well experiment}\label{sec:well}

\begin{figure}%{l}{0.5\textwidth}
% \begin{figure}[H]
\centering
    \includegraphics[scale=0.6]{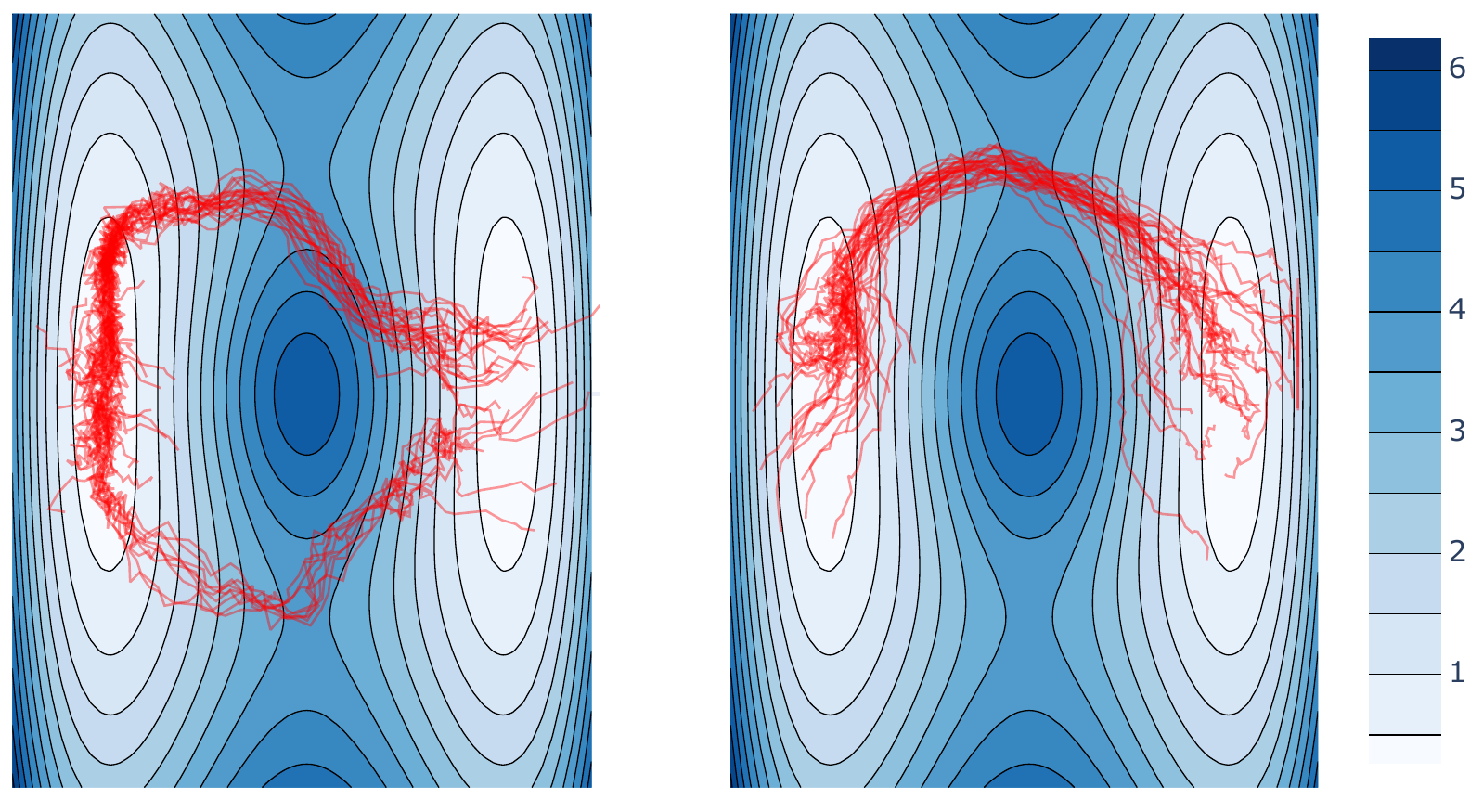}
% \end{figure}
\caption{IPML forward trajectories on potential well. Exponentiated kernel on the left and Exponential Quadratic on the right.}
    \label{fig:well}
\end{figure}

We used the following potential to model the double wells:
\begin{align}
    U(x,y) = \frac{5}{2}(x^2-1)^2+y^2 +  \frac{1}{\delta}\exp\left(-\frac{  x^2 +y^2}{\delta}\right),
\end{align}
furthermore we used the boundary distributions:
\begin{align*}
    \pi_0 \sim \calN\left(\begin{pmatrix} -1 \\
    0\\
    \end{pmatrix} ,\begin{pmatrix} 0.0125
 & 0 \\
    0 & 0.15\\
    \end{pmatrix}\right),\;\; \pi_1 \sim 
    \calN\left(\begin{pmatrix} 1 \\
    0\\
    \end{pmatrix} ,\begin{pmatrix} 0.0125
 & 0 \\
    0 & 0.15\\
    \end{pmatrix}\right)
\end{align*}
Where we selected the boundary distributions to be located at the wells and visually follow a similar spread/curvature as the contour of the well. For these experiments we use $\gamma=3$. As highlighted previously $\pi_0^{\Q_0^\gamma}$ is a free parameter which in this section we set as:
\begin{align}
\pi_0^{\Q_0^\gamma} \sim \calN\left(\begin{pmatrix} 0 \\
    0\\
    \end{pmatrix} ,\begin{pmatrix} 0.5
 & 0 \\
    0 & 0.5\\
    \end{pmatrix}\right),
\end{align}
we select this as the marginal prior since its samples cover the space over which the wells is defined well, and helps estimate the backwards drift of the prior more accurately. As mentioned in the paper, using different kernesl as well as lengthscales resulted in slightly different behavior on the well experiment where the trajectory would not split as illustrated in figure \ref{fig:well}. The splitting pattern was obtained using an exponential kernel with a lower lengthscale ($0.25$) underlying the need for careful consideration of the kernel and its parameters.  

\subsubsection{Discussion on Kernel Choice}

We can observe that when using EQ kernel, the trajectories do not cross both passes but instead choose one. We conjecture that this is due to the fitted predictive means having a preference for simpler functions that send nearby particles in the same direction. We empirically verify this conjecture by experimenting with alternative kernels to EQ (i.e. exponential) that model a wider class of functions. We believe that these alternative kernels are less prone to discouraging splitting since, when viewed within the kernel ridge regression framework, they span an RKHS of functions that do not have the smoothness constraint of the EQ kernel. 
\begin{figure}[t!]
    \centering
    \includegraphics[scale=0.65]{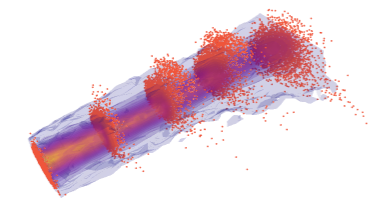}
    \caption{Fitted SBP distribution on the start and end data slices with Brownian prior on single cell human Embryo data from (Tong et al. 2020). Observations depicted as red point clouds. See section \ref{sec:cells} for experimental details.}
    \label{fig:cell_volume}
\end{figure}
\subsection{Cell experiment}
The data used for this experiment was taken directly from the github repository published by \cite{tong2020trajectorynet}. PCA is first applied to the data and the first 5 components are selected. We used the code available at \url{https://github.com/KrishnaswamyLab/TrajectoryNet/tree/master/TrajectoryNet} to reproduce the performance of TrajectoryNet as well as the optimal transport baseline. 
\subsection{Motion experiment}\label{app:mot}

\begin{figure}
    \centering
    \includegraphics[scale=0.33]{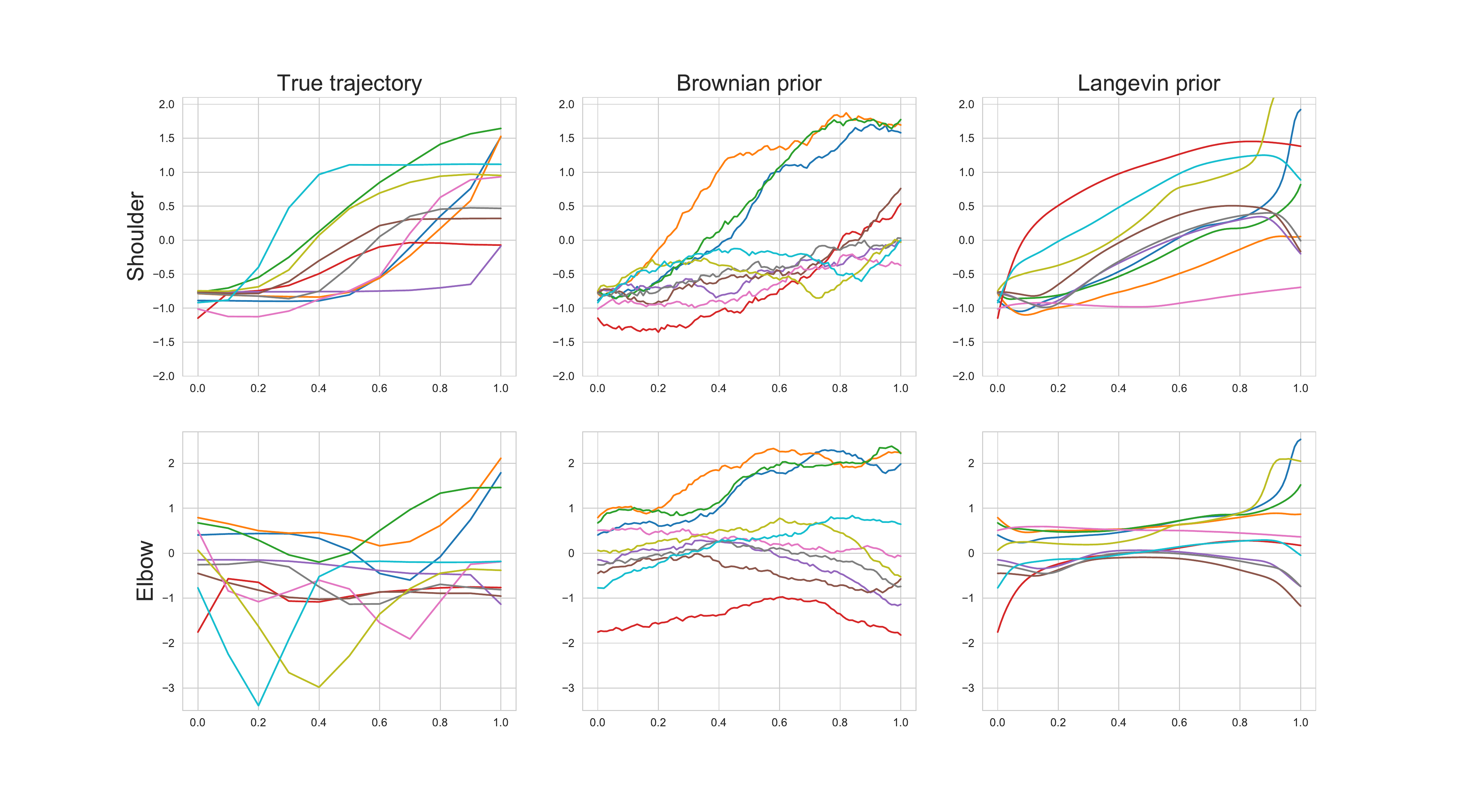}
    \caption{Trajectories of the shoulder and arm oriented angle sensor through the motion. The two plots on the right demonstrate IPML's fit using a Brownian and Langevin prior.}
    \label{fig:motion_plot}
\end{figure}

For this experiment we use a second order SDE prior (i.e. Langevin dynamics). This can be encoded in first order SDE using the companion/controllable form of a dynamical system:

\begin{align}
d\begin{pmatrix}
\rvx(t)\\
\rvv(t)
\end{pmatrix} = \begin{pmatrix}
\rvv(t)\\
 \mK\rvx(t)-\beta\rvv(t)
\end{pmatrix}dt + \gamma \begin{pmatrix}
\mathbf{0}_{d}\\
\mathbf{1}_d
\end{pmatrix} d\rvw(t)
\end{align}
For simplicity we set the dampening factor $\beta$ to $0$ and $\gamma=0.5$. As for the spring matrix we explored the following decoupled and coupled forms:
\begin{align}
    \mK = k \mathbb{I} \quad \mK = k\frac{3}{2}\mathbb{I} - \frac{k}{2} \mathbf{1} \mathbf{1}^{\top}
\end{align}
Where $k$ defines the frequency of the solution via $\omega = \sqrt{k/m}$ in our setting we ignore the mass (imagine is at absorbed in the matrix $\mK$ and work directly with setting $\omega$. Since the movement we select moves from rest to a high position and does not return to its starting point / or repeat itself we set $\omega\approx\pi$ that corresponds to $0.5$Hz which describes half a revolution. We visually confirmed with the true trajectories that this seems like a reasonable prior.

For the Brownian motion prior that we use as a baseline comparison in Figure \ref{app:mot} we explored several values of $\gamma$ and selected the best performing one (visually) of $\gamma=0.3$. Note we do pick a smaller value of $\gamma$ for the Brownian motion and this is partly to compensate that the Langevin prior trajectories will always look smooth by construction thus we pick a not so noisy Brownian motion that still gives good results. 

Note that the volatility term here is singular and in our mocap application setting it is exactly the vector:
\begin{align}
   \gamma  \begin{pmatrix}
    \mathbf{0}_d \\
    \mathbf{1}_d
    \end{pmatrix}
\end{align}
Where the dimension $d$ is given by the number of sensors we use to fit the motion. A careful reader may wonder weather the SBP machinery still applies to such a sparsely structured diffusion and in fact it does. Following Theorem 4 of \cite{sottinen2008application} we can see how the Radon-Nikodym derivative is finite and can be expressed in an almost identical fashion to the original controlled SBP formulation and thus justifying the existence of the SBP in this scenario as well as the application of the IPML algorithm and its derivatives.

\subsection{Computational resources}\label{app:comp}
\paragraph{Computational costs:}

The main computational cost of IPML comes from fitting a Gaussian Process to model the drift (DriftFit). Assuming the cardinality of  $|\pi_0|=|\pi_1|=N$ and a discretization factor $\Delta_t$, at each iteration, IPML requires fitting a Gaussian Process on $\frac{N}{\Delta_t}$ samples. It is a well-known fact that GPs have a cubic time complexity making the costs of the DriftFit subroutine $O((\frac{N}{\Delta_t})^3)$. This could be scaled down by using GP approximations (Nystrom) and was successful in our preliminary results. The computational costs of the SDESolve sub-routine in comparison is $O((\frac{N}{\Delta_t})^2)$ due to the GP predictions costs. Finally, the memory complexity of the algorithm is $O((\frac{N}{\Delta_t})^2)$ due to the GP fitting in DriftFit. The running time of IPML on the machine described in the section below for the well experiment is around 5 minute. For the cell experiment, the running time is between 1 and 2 hours depending on the discretization factor chosen. 
\paragraph{Infrastructure used:} The experiments are performed on Compute Canada clusters. Specifically, they ran on a CPU cluster composed of Intel CPUs. Each node contains 20 Intel Skylake cores (2.4GHz, AVX512), for a total of 40 cores per node and 202GB of RAM. The computational costs for the toy experiments (Well and 1D)  was quite small, however, the RAM consumption of the cell experiment was significant due to the use of a Gaussian Process.

% \end{paracol}

\end{document}